\documentclass[10pt,twocolumn]{IEEEtran}

\usepackage{amsmath,amssymb,amsthm}
\usepackage[utf8x]{inputenc}
\usepackage{graphicx}
\usepackage{indentfirst}
\usepackage{cmap}
\usepackage{ifthen}
\usepackage{tikz}
\usepackage{algorithm}
\usepackage{algorithmic}
\usepackage{pifont}
\usepackage{url}
\usepackage{mathtools}
\usepackage{algorithm,algorithmic}
\usepackage{hyperref}
\usepackage{mathrsfs}
\usepackage{multirow}
\RequirePackage{etex}
\usepackage{subfigure}
\usepackage{pgf}
\usepackage{tikz}
\usetikzlibrary{arrows,automata}
\usepackage{enumitem}

\renewcommand{\le}{\leqslant}
\renewcommand{\ge}{\geqslant}
\renewcommand{\hat}{\widehat}
\renewcommand{\emptyset}{\varnothing}

\newtheorem{theorem}{Theorem}[section]
\newtheorem{corollary}{Corollary}[theorem]
\newtheorem{lemma}[theorem]{Lemma}
\newtheorem{assumption}{Assumption}

\newtheorem{proposition}[theorem]{Proposition}

\DeclareMathOperator*{\argmin}{arg\,min}

\def\jzh#1{{\color{red}#1}} 

\usepackage{pgfplotstable}
\usepackage{pgfplots}
\pgfplotsset{
	table/search path={plot_figures},
}
\usepackage{tikz}
\usepackage{xr}
\usetikzlibrary{external}
\usepackage{cite}
\usepackage{subfigure}
\pgfdeclarelayer{bg}    
\pgfsetlayers{bg,main}  
\pgfplotsset{compat=1.14}

\title{\LARGE \bf
	A General Framework for Distributed Inference \\ with Uncertain Models
}

\author{James Z. Hare, C\'esar A. Uribe, Lance Kaplan, \emph{Fellow, IEEE}, and Ali Jadbabaie \emph{Fellow, IEEE} 
	\thanks{This research was sponsored by the DARPA Lagrange, Vannevar Bush Fellowship, and OSD LUCI programs. 
	}
	\thanks{J.Z.H. and L.K. (\textit{\{james.z.hare.civ, lance.m.kaplan.civ\}@mail.mil}) are with the U.S. Army Research Laboratory, Adelphi, MD. C.A.U and A.J. are with the Laboratory for Information and Decision Systems (LIDS), and the Institute for Data, Systems, and Society (IDSS),
		Massachusetts Institute of Technology, Cambridge, MA
		(\textit{\{cauribe,jadbabai\}@mit.edu}).  }%
}
\begin{document}

\maketitle

\begin{abstract}
This paper studies the problem of distributed classification with a network of heterogeneous agents. The agents seek to jointly identify the underlying target class that best describes a sequence of observations. The problem is first abstracted to a hypothesis-testing framework, where we assume that the agents seek to agree on the hypothesis (target class) that best matches the distribution of observations.  Non-Bayesian social learning theory provides a framework that solves this problem in an efficient manner  by allowing the agents to sequentially communicate and update their beliefs for each hypothesis over the network. Most existing  approaches assume that agents have access to exact statistical models for each hypothesis. However, in many practical applications, agents learn the likelihood models based on limited data, which induces uncertainty in the likelihood function parameters. In this work, we build upon the concept of \emph{uncertain models} to incorporate the agents' uncertainty in the likelihoods by identifying a broad set of parametric distribution that allows the agents' beliefs to converge to the same result as a centralized approach. Furthermore, we empirically explore extensions to non-parametric models to provide a generalized framework of uncertain models in non-Bayesian social learning.   
\end{abstract}

\section{Introduction}

Non-Bayesian social learning  provides a scalable approach for distributed inference of boundedly rational agents with heterogeneous sensing modalities interacting over a network~\cite{JMST2012}. In this setting, agents receive a stream of partial observations conditioned on true \textit{state of the world.} Each agent uses its private observations and network communications to construct a set of beliefs on a finite set of possible states of the world or hypotheses. The agents' goal is to jointly identify a hypothesis that best explains the set of observations, resulting in an estimate of the true state of the world.

Collaboration happens when an agent combines its neighbors' beliefs (a normalized aggregated likelihood of the state of the world) at each time step via a fusion step. The agent then updates the combined beliefs (prior) with the likelihood of their most recent private observation, resulting in a posterior belief. This enables a scalable learning approach that does not require \emph{a priori} knowledge of the network structure or the agents' past observations, which avoids the ``double counting'' problem faced in Bayesian social learning \cite{GK2003, ADLO2011, KT2013, RJM2017}. 

Several social learning (fusion) rules have been proposed in the literature, including weighted averages~\cite{JMST2012, SJ2013}, geometric averages~\cite{RT2010,RMJ2014, LJS2018}, constant elasticity of substitution models~\cite{MTJ18}, and minimum operators~\cite{MRS2019_ACC, MRS2019_CDC}. These learning rules have been applied to undirected/directed graphs, time-varying graphs~\cite{NOU2015,NOU2017}, weakly-connected graphs~\cite{SYS2017, SYS2018}, agents with increasing self-confidence~\cite{UJ2019}, compact hypothesis sets~\cite{NOU2017_compact}, and under adversarial attacks~\cite{BT2018, SV2018, VSV2019, HULJ2019}. Each approach presents a variation of one of the above learning rules and provides theoretical guarantees (asymptotically) that the agents will \textit{learn} the true state of the world. 

The fundamental assumption in these  approaches is that the parameters of the likelihood models for each hypothesis are known precisely. For example, each agent may collect a large set of labeled training data or \emph{prior evidence} for each hypothesis, which allows them to identify the precise parameters. However, collecting prior evidence is costly, and often training happens with limited data, which can lead to inaccurate inferences~\cite{CG1999}.

Incorporating uncertainty into the statistical models has been studied from a non-Bayesian perspective in the fields of possibility theory~\cite{DP2012}, probability intervals~\cite{W1997}, and belief functions~\cite{S1976, SK1994} by expanding beyond probability theory to identify fixed intervals of uncertainty for each parameter of the likelihood function. Other modeling approaches follow a Bayesian perspective, which models the uncertainty in the parameters as a second-order probability density function~\cite{J2018}. This second-order probability density function is typically a conjugate prior of the likelihood model, allowing for a mathematically convenient approach to computing the posterior distribution of the parameters conditioned on the prior evidence. Then, the~\emph{uncertain likelihood function} can be computed as the posterior predictive distribution~\cite{R1984}, which marginalizes the likelihood over the unknown parameters.

Recently proposed approaches have incorporated uncertainty into non-Bayesian social learning theory through the concept of \emph{uncertain models} \cite{HUKJ2020_TSP, HULJ2019, UHLJ2019, HUKJ2020_Gauss, HUKJ2020_ICASSP}. An uncertain model consists of an \emph{uncertain likelihood ratio} as the likelihood model of each agent, which tests the consistency of the prior evidence with a stream of private observations collected in the testing phase. This ratio aims to identify whether the prior evidence and the private observations are drawn from the same/different distributions and consists of a posterior predictive distribution normalized by a prior predictive distribution. 

Initially, uncertain models were developed to handle categorical observations and prior evidence drawn from multinomial distributions~\cite{HUKJ2020_TSP} and were later extended for data drawn from univariate Gaussian distributions~\cite{HUKJ2020_Gauss}. These models were implemented into a non-Bayesian social learning rule and studied for static and time-varying graphs~\cite{UHLJ2019}. Additionally, the social learning rule with uncertain models was adjusted to handle communication constrained environments~\cite{HUKJ2020_ICASSP} and adversarial agents~\cite{HULJ2019}. These works showed that the asymptotic beliefs of each agent in the network converge to a weighted geometric average of their uncertain likelihood ratios when using a geometric average social learning rule, resulting in a one-to-one relation with the centralized solution. Furthermore, when the agents become certain in their models, i.e., know the likelihood parameters precisely, the agents can learn the true state of the world, providing a consistent result with traditional social learning. 

In this work, we build upon~\cite{HUKJ2020_TSP, HUKJ2020_Gauss} by identifying a broad set of parametric distributions that enable learning with uncertain models. Additionally, we identify conditions that allow the agents to include \emph{model uncertainty} \cite{HW2019}, i.e., the true statistical model is not within the parametric family of distributions. The uncertain models are implemented into a non-Bayesian social learning rule and show that the results are consistent and that the works presented in \cite{HUKJ2020_TSP, HUKJ2020_Gauss} are special cases of the uncertain models presented herein. Additionally, we provide an algorithmic representation of how to implement uncertain models in a practical setting for continuous and discrete observations. Finally, we extend the uncertain models to handle data drawn from non-parametric distributions and empirically show that the main results hold. 

This paper is organized as follows. Section~\ref{sec:pfam} provides the problem formulation, while Section~\ref{sec:gum} presents the general uncertain models. Section~\ref{sec:nbsl_um} implements the uncertain models into non-Bayesian social learning and provides the asymptotic properties of the beliefs. Section~\ref{sec:implement} shows the algorithmic steps to implement uncertain models with examples of data drawn from both discrete and continuous distributions. Section~\ref{sec:example_non} provides a non-parametric framework for uncertain models and Section~\ref{sec:results} includes a numerical analysis of both parametric and non-parametric models. 
Finally, we conclude the paper in Section~\ref{sec:con} and discuss future work.   

\textbf{Notation:} Bold symbols represent a vector/matrix, while non-bold symbols represent its element. All vectors are column vectors unless specified. The indexes $i$ and $j$ represent agents and $t$ represents time. We abbreviate the terminology independent identically distributed as i.i.d.. We use $[\mathbf{A}]_{ij}$ to represent the entry of matrix $\mathbf{A}'s$ $i$th row and $j$th column. The empty set is denoted as $\emptyset$.We denote the Kullback-Liebler (KL) divergence as 
\begin{eqnarray}
D_{KL}(p(x)\|q(x)) = \int p(x)\log\left(\frac{p(x)}{q(x)} \right)dx,
\end{eqnarray}
where $p(x)$ and $q(x)$ are two probability distributions over $x$. 

\section{Problem Formulation} \label{sec:pfam}

\subsection{Classification as a Non-Bayesian Learning Problem}

We consider a group of $m$ heterogeneous agents connected over a network with the task of classifying a source of information into one of $M$ possible classes $\boldsymbol{\Theta}=\{\theta_1,...,\theta_M\}$. An agent $i$ collects sensor measurements about the source modeled as a sequence of realizations of a random variable distributed according to an unknown probability distribution $Q^i$, conditioned on the source being of a class $\theta^*$. Note that heterogeneity of sensor modalities available at the agents implies that measurements of different agents could be drawn from different random distributions, i.e., $Q^i$ might be different from $Q^j$ for $j\ne i$ due to different sensing phenomenology. 

This classification problem can be abstracted into a distributed hypothesis testing framework, where each target class $\theta\in\boldsymbol{\Theta}$ for each agent $i$ is represented as a parametrized distribution $P^i(\cdot|\boldsymbol{\phi}_\theta^i)$, where $\boldsymbol{\phi}_\theta^i$ is the set of parameters known by the agent.\footnote{We assume that the parametrized distribution for each agent may vary due to their heterogeneous sensing capabilities.} For example, if the agent considers that each class $\theta$ is a Gaussian distribution, the set of parameters are $\boldsymbol{\phi}_\theta^i=\{m_\theta^i,(\lambda_\theta^i)^{-1}\}$, where $m_\theta^i$ and $\lambda_\theta^i$ are the mean and precision, respectively. Under this setup, the common objective of the set of agents is to solve the following optimization problem jointly
\begin{align}\label{eq:main}
    \hat{\theta} = \argmin_{\theta \in \boldsymbol{\Theta}} \sum_{i=1}^m D_{KL} \big(Q^i \| P^i(\cdot|\phi^i_\theta)\big),
\end{align}
where $Q^i$ is the unknown distribution of the observations conditioned on the target class. Later in the next subsection we will describe the sources of uncertainty for the non-Bayesian social learning problem.

Indeed, finding a solution of \eqref{eq:main} implies finding a class whose conditional likelihood function is statistically similar to the distribution of the observations conditioned on the true class. To achieve this objective, each agent $i$ utilizes their sensing device to receive a stream of i.i.d. observations over discrete time $t\ge 1$, $\boldsymbol{\omega}_{1:t}^i=\{\omega_1^i,...,\omega_t^i\}$, where each $\omega_\tau^i$ for $\tau\in\{1,...,t\}$ is drawn from the ground truth distribution $Q^i$. In order to achieve a unique solution to~\eqref{eq:main}, we impose the following assumption.   
\begin{assumption}\label{assum:unique}
The intersection of the optimal hypotheses for each individual agent has a unique element $\theta^*$, i.e., the set $\bigcap_{i \in \mathcal{M}} \boldsymbol{\Theta}^*_i = \{\theta^*\}$, where 
\begin{align}\label{eq:main_local}
    \boldsymbol{\Theta}^*_i = \argmin_{\theta \in \boldsymbol{\Theta} }  D_{KL} \big(Q^i \| P^i(\cdot|\phi^i_\theta)\big).
\end{align}
is the set of indistinguishable hypotheses of agent $i$. 
\end{assumption}

Assumption~\ref{assum:unique} states that each agent could potentially solve a local problem using local information only. However, there is no guarantee that the solution to the local problem is unique, i.e., one individual agent might not have the capability to identify $\theta^*$. However, we assume that collectively the network can collaborate to identify a unique hypothesis that represents the ground truth. 
Therefore, $\theta^*$ is the unique solution to~\eqref{eq:main}.

We assume agents interact over a graph $\mathcal{G}=\{\mathcal{M},\mathcal{E}\}$, where $\mathcal{M}$ is the set of agents and $\mathcal{E}$ is the set of edges connecting the agents. If $(i,j)\in\mathcal{E}$, then agents $i$ and $j$ can communicate to each other. We denote agent $i$'s set of neighbors as $\mathcal{M}^i=\{j|(i,j)\in\mathcal{E}, \forall j\in\mathcal{M}\}$ and each edge is assumed to be weighted and modeled as an adjacency matrix~$\mathbf{A}$. Furthermore, we assume the following properties. 

\begin{assumption}\label{assum:graph}

               The graph $\mathcal{G}$ and matrix $\mathbf{A}$ are such that:

               \begin{enumerate}

                              \item $\mathbf{A}$ is doubly-stochastic with $\left[\mathbf{A}\right]_{ij} = a_{ij} > 0$ for  $i\ne j$ if and only if $(i,j)\in E$.

                              \item $\mathbf{A}$ has positive diagonal entries, $a_{ii}>0$ for all $i \in \mathcal{M} $.

                              \item The graph $\mathcal{G}$ is connected.

               \end{enumerate}

\end{assumption}

Assumption~\ref{assum:graph} states that the adjacency matrix is ergodic, i.e., aperiodic and irreducible, and it is a 
common assumption 
in the literature \cite{NOU2017}. This allows every agent to aggregate their local information throughout the entire network. Note that our assumptions are applicable for either directed or undirected graphs $G$. 

The theory of non-Bayesian social learning provides a framework that enables the agents to jointly solve~\eqref{eq:main} in a distributed fashion. Here, each agent $i$ holds a belief $\mu_t^i(\theta)$ at each time step $t$, which represents the probability that the target class $\theta$ is the ground truth. We denote the set of beliefs for each agent $i$ at time $t$ as $\boldsymbol{\mu}_t^i=\{\mu_t^i(\theta)\}_{\forall \theta\in\boldsymbol{\Theta}}$. As seen in Figure~\ref{fig:net_example}, at each time $t\ge 1$, each agent $i$ updates their beliefs using a social learning rule that consists of fusing their neighbors beliefs from the previous time step, $\boldsymbol{\mu}_{t-1}^{\mathcal{M}^i}=\{\boldsymbol{\mu}_{t-1}^j\}_{\forall j\in\mathcal{M}^i}$, and scaling the combined beliefs with the likelihood of a new observation $P^i(\omega_t^i|\boldsymbol{\phi}_\theta^i)$. One common belief update rule in non-Bayesian social learning is based on a geometric average of the beliefs and is defined as follows, 
\begin{align} \label{eq:trad_belup}
    \mu_{t}^i(\theta) \propto P^i(\omega_t^i|\boldsymbol{\phi}_\theta^i)\prod_{j\in\mathcal{M}^i}\mu_{t-1}^j(\theta)^{[\mathbf{A}]_{ij}}.
\end{align}
Once the beliefs are updated, the agent transmits $\boldsymbol{\mu}_t^i$ to their neighbors and the process is repeated. The update rule in~\eqref{eq:trad_belup} has the property that the belief of the target class $\theta$ that solves the optimization problem~\eqref{eq:main} will converge to $1$ almost surely for every agent, while the remaining beliefs converge to $0$, allowing for the agents to learn the ground truth~\cite{nedic2017distributed}. 

\subsection{Uncertainty in Non-Bayesian Social Learning} \label{sec:types_of_uncertainty}

The current framework of non-Bayesian social learning theory does not account for two types of uncertainty that are commonly found in practical applications: 1) \emph{Epistemic} uncertainty in the parameters $\boldsymbol{\phi}_\theta^i$ of the likelihood models for each $\theta$, and 2) \emph{Model} uncertainty associated with the family of distributions $P^i(\cdot|\boldsymbol{\phi}_\theta^i)$. 

Recall that in the non-Bayesian social learning setup, each agent is assumed to have a parametric family of distributions corresponding to the conditional distributions for each possible target class. However, such models are usually built from training data. Epistemic uncertainty arises when the agents have a limited/finite amount of training data, $\mathbf{r}_\theta^i=\{r_{\theta,k}^i\}_{k=1,...,|\mathbf{r}_\theta^i|}$, for each target class $\theta\in\boldsymbol{\Theta}$. If the agent estimates the parameters of the likelihood model using $\mathbf{r}_\theta^i$, there is a probability greater than $0$ that a belief for a hypothesis $\hat{\theta}\ne\theta^*$ updated using~\eqref{eq:trad_belup} will converge to $1$, while the ground truth belief converges to $0$. Therefore, it is necessary to adjust the belief update rule~\eqref{eq:trad_belup} to account for the epistemic uncertainty in the parameters of the likelihood models by incorporating \emph{uncertain} models (see Section~\ref{sec:gum}). 

Model uncertainty arises naturally when the underlying physics or background knowledge of the ground truth distribution is unknown or partially known to the agents, resulting in misspecified likelihood models~\cite{HW2019}. Additionally, when the agents have finite prior evidence, the uncertain models may be within a family of distributions such that there is not be a $\theta\in\boldsymbol{\Theta}$ such that $Q^i = P^i(\cdot|\phi^i_\theta)$~\cite{W2013}~\cite[p.~88]{gelman2013}, requiring the agents to adjust the belief update rule~\eqref{eq:trad_belup} to handle this uncertainty. 

For example, after acquiring a finite number of samples $\mathbf{r}_\theta^i$ for each hypothesis $\theta$, it is reasonable that the agent $i$ may assume that the data is distributed according to a Normal distribution with unknown mean $m^i_\theta$ and standard deviation $\sigma^i_\theta$ such that $\boldsymbol{\phi}_\theta^i = \{m^i_\theta, \sigma_\theta^i\}$. However, it is entirely possible that the underlying ground truth distribution $Q^i$ is a student-$t$ distribution or some otherwise arbitrary distribution with mean and variance given by $\boldsymbol{\phi}^i_\theta$. Nevertheless, as long as there is a unique set of parameters $\boldsymbol{\phi}^i_{\theta^*} = \{m^i_{\theta^*}, \sigma_{\theta^*}^i\}$ that minimizes the KL divergence between the ground truth distribution $Q^i$ and the parameterized distribution $P^i(\cdot|\phi^i_\theta)$ for hypothesis $\theta=\theta^*$, learning can occur.\footnote{The parameters $\phi^i_{\theta^*}$ associate to the Maximum Likelihood estimate of the parameters given an infinite amount of observations $\omega^i$. }
 
\textit{The overall goal of this work is to incorporate both epistemic and model uncertainties into non-Bayesian social learning theory and provide theoretical guarantees on the convergence properties of the uncertain beliefs.}

\begin{figure}[t!]
    \centering
    \includegraphics[width=\columnwidth]{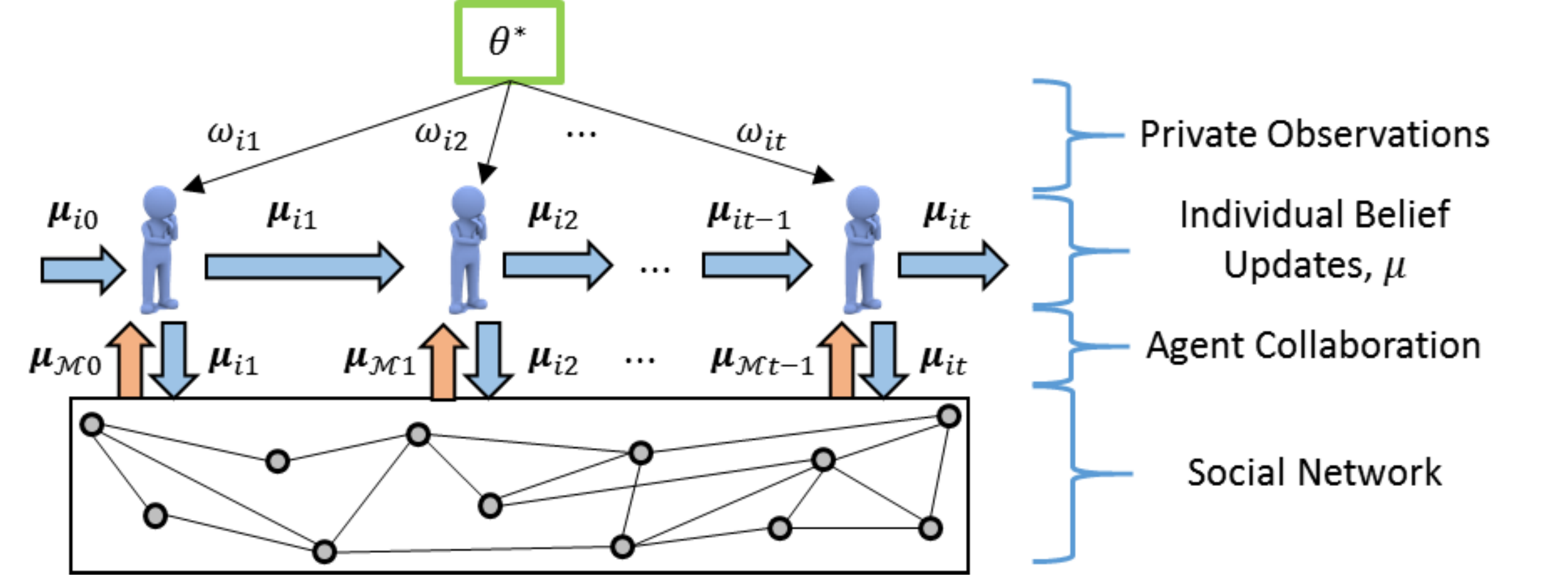}
    \caption{Example of social learning for hypothesis testing. The world selects a state $\theta^*$, and each agent sequentially observes realizations $\omega_{t}^i$ of a random variable whose probability distribution is conditioned on $\theta^*$. Each agent constructs and shares a set of beliefs $\boldsymbol{\mu}_{it}$ over a social network to cooperatively select the hypothesis that best describes the set of observations.} 
    \label{fig:net_example} \vspace{-16pt}
\end{figure}

\section{General Uncertain Models} \label{sec:gum}

In this section, we present a general class of \emph{uncertain models} and their asymptotic properties for a single agent that has collected a limited amount of training data for each class~\cite{HUKJ2020_TSP, UHLJ2019, HUKJ2020_Gauss}. For simplicity and ease of presentation, we drop the subscript $i$ in our notation for the remainder of this section. Later in Section~\ref{sec:nbsl_um} we extend this to the distributed setup.

\subsection{The Uncertain Likelihood Ratio}
The first step in deriving the uncertain likelihood ratio is to identify the parameters $\boldsymbol{\phi}_\theta$ of the likelihood model $P(\cdot|\boldsymbol{\phi}_\theta)$. Consider that the agent has collected a set of prior evidence $\mathbf{r}_\theta=\{r_{k,\theta}\}_{k=1,...,|\mathbf{r}_\theta|}$ (training data) for each hypothesis $\theta\in\boldsymbol{\Theta}$, which consists of a set of i.i.d. samples drawn from the distribution $Q_\theta^i$, where the amount of prior evidence $|\mathbf{r}_\theta|$ collected for each hypothesis may vary. Then, we make the following assumption about the likelihood models that allows the agent to model their epistemic uncertainty in the parameters $\phi_\theta$. 
\begin{assumption} \label{assum:conjugate}
The set of sufficient statistics of the family of distributions $P(\cdot|\boldsymbol{\phi})$ is finite.
\end{assumption}

Assumption~\ref{assum:conjugate} implies that there exists a conjugate distribution to $P(\cdot|\boldsymbol{\phi})$ that models the distribution of the parameters $\boldsymbol{\phi}$~\cite{D2005}. Therefore, we can take a Bayesian approach to modeling the epistemic uncertainty in the parameters by utilizing the prior evidence $\mathbf{r}_\theta$ to compute the posterior distribution 
\begin{align} \label{eq:cp}
    f(\boldsymbol{\phi}|\psi(\mathbf{r}_\theta)) = \frac{1}{Z(\mathbf{r}_\theta)} P(\mathbf{r}_\theta|\boldsymbol{\phi})f_0(\boldsymbol{\phi}),
\end{align}
where $Z(\mathbf{r}_\theta)=\int_{\boldsymbol{\Phi}} P(\mathbf{r}_\theta|\boldsymbol{\phi})f_0(\boldsymbol{\phi}) d\boldsymbol{\phi}$ is the normalization factor; $P(\mathbf{r}_\theta|\boldsymbol{\phi})$ is the assumed family of distributions; and $f_0(\boldsymbol{\phi})=f(\boldsymbol{\phi}|\psi(\emptyset))$ is the conjugate prior of $P(\cdot|\boldsymbol{\phi})$ with hyperparameters $\psi(\emptyset)$ chosen s.t. the prior is uninformative with full support over the parameter space $\boldsymbol{\Phi}$. Note that conjugate distributions allow for the agent to simply update the hyperparameters of the posterior as a function of the amount of data present, for example the prior evidence $\mathbf{r}_\theta$ as in~\eqref{eq:cp}. Examples of the update equations for specific parametric distributions are presented in Section~\ref{sec:implement}.  

Then, because the parameters are only known within a distribution $f(\boldsymbol{\phi}|\psi(\mathbf{r}_\theta))$, the agent must use a posterior predictive distribution in lieu of the likelihood model to form the uncertain likelihood, i.e., the uncertain likelihood of a set of measurements $\boldsymbol{\omega}_{1:t}$ for $t\ge 1$ is defined as
\begin{align} \label{eq:ul}
    \hat{P}(\boldsymbol{\omega}_{1:t}|\mathbf{r}_\theta) = \int_{\boldsymbol{\Phi}}P(\boldsymbol{\omega}_{1:t}|\boldsymbol{\phi})f(\boldsymbol{\phi}|\psi(\mathbf{r}_\theta)) d\boldsymbol{\phi}.
\end{align}

Normally, the agent would then construct a likelihood ratio test by normalizing the uncertain likelihood over the set of hypotheses. However, when the amount of prior evidence varies between hypotheses, i.e., $|\mathbf{r}_\theta|\ne |\mathbf{r}_{\hat{\theta}}|$ for some $\theta \ne \hat{\theta}$, the uncertain likelihoods become incommensurable since they do not have a common standard of measurement, as shown in~\cite{HUKJ2020_TSP}. Instead, the agents evaluate each hypothesis independently and normalize the uncertain likelihood by the model of complete ignorance to form the \emph{uncertain likelihood ratio} as follows
\begin{align} \label{eq:ULR}
    \Lambda_\theta(t) = \frac{\hat{P}(\boldsymbol{\omega}_{1:t}|\mathbf{r}_\theta)}{\hat{P}(\boldsymbol{\omega}_{1:t}|\mathbf{r}_\theta = \emptyset)},
\end{align}
where $\hat{P}(\boldsymbol{\omega}_{1:t}|\mathbf{r}_\theta = \emptyset)=\int_{\boldsymbol{\Phi}}P(\boldsymbol{\omega}_{1:t}|\boldsymbol{\phi})f_0(\boldsymbol{\phi}) d\boldsymbol{\phi}$. The model of complete ignorance is defined as a prior predictive distribution of the measurement sequence conditioned on the agent collecting zero prior evidence. 

\subsection{The Uncertain Likelihood Update}
Thus far, we have derived the uncertain likelihood ratio based on an agent collecting a set of measurements $\boldsymbol{\omega}_{1:t}^i$ up to time $t\ge 1$. However, in the typical setting, the agent will sequentially collect a single measurement at each time step, requiring that uncertain likelihood ratio to be decomposed into a recursive function. 

This is achieved by first expressing the uncertain likelihood~\eqref{eq:ul} as follows,
\begin{align} \label{eq:ul_recursive}
    \hat{P}&(\boldsymbol{\omega}_{1:t}|\mathbf{r}_\theta) = \int_{\boldsymbol{\Phi}} P(\omega_{t}|\boldsymbol{\phi})\prod_{\tau=1}^{t-1}P(\omega_{\tau}|\boldsymbol{\phi})f(\boldsymbol{\phi}|\psi(\mathbf{r}_\theta))d\boldsymbol{\phi} \nonumber \\
    & = \hat{P}(\boldsymbol{\omega}_{1:t-1}|\mathbf{r}_\theta)\int_{\boldsymbol{\Phi}} P(\omega_{t}|\boldsymbol{\phi})f(\boldsymbol{\phi}|\psi(\{\boldsymbol{\omega}_{1:t-1},\mathbf{r}_\theta\})) d\boldsymbol{\phi}, 
\end{align}
where the first line is due to the i.i.d. observations, while the second line is achieved from Bayes' rule by substituting the following 
\begin{align} 
    &\prod_{\tau=1}^{t-1}P(\omega_{\tau}|\boldsymbol{\phi})f(\boldsymbol{\phi}|\psi(\mathbf{r}_\theta)) \nonumber \\
    & = f(\boldsymbol{\phi}|\psi(\{\boldsymbol{\omega}_{1:t-1},\mathbf{r}_\theta\}))\int_{\boldsymbol{\Phi}} P(\boldsymbol{\omega}_{1:t-1}|\boldsymbol{\phi})f(\boldsymbol{\phi}|\psi(\mathbf{r}_\theta))d\boldsymbol{\phi} \nonumber \\
    & = f(\boldsymbol{\phi}|\psi(\{\boldsymbol{\omega}_{1:t-1},\mathbf{r}_\theta\}))\hat{P}(\boldsymbol{\omega}_{1:t-1}|\mathbf{r}_\theta),
\end{align}
and the fact that $\hat{P}(\boldsymbol{\omega}_{1:t-1}|\mathbf{r}_\theta)$ is a constant for a given sequence of observations $\boldsymbol{\omega}_{1:t-1}$. Similarly, we can write the model of complete ignorance as follows,
\begin{align} \label{eq:ig_recursive}
    &\hat{P}(\boldsymbol{\omega}_{1:t}|\mathbf{r}_\theta=\emptyset) \nonumber \\ &= \hat{P}(\boldsymbol{\omega}_{1:t-1}|\mathbf{r}_\theta=\emptyset)\int_{\boldsymbol{\Phi}} P(\omega_t|\boldsymbol{\phi})f(\boldsymbol{\phi}|\psi(\boldsymbol{\omega}_{1:t-1}))d\boldsymbol{\phi}.
\end{align}
Then, we can rewrite the uncertain likelihood ratio in the following recursive form,
\begin{align} \label{eq:url_recursive}
    \Lambda_\theta(t) &= 
     \ell_\theta(\omega_t) \Lambda_\theta(t-1),
\end{align}
where we define the \emph{uncertain likelihood update} as 
\begin{align} \label{eq:ulu}
    \ell_\theta(\omega_t) = \frac{\int_{\boldsymbol{\Phi}} P(\omega_{t}|\boldsymbol{\phi})f(\boldsymbol{\phi}|\psi(\{\boldsymbol{\omega}_{1:t-1},\mathbf{r}_\theta\})) d\boldsymbol{\phi}}{\int_{\boldsymbol{\Phi}} P(\omega_t|\boldsymbol{\phi})f(\boldsymbol{\phi}|\psi(\boldsymbol{\omega}_{1:t-1}))d\boldsymbol{\phi}}.
\end{align} 

\subsection{Asymptotic Properties of the Uncertain Likelihood Ratio}

Next, we compute the asymptotic properties of the uncertain likelihood ratio. First, we must state the following assumptions regarding the statistical models for each hypothesis $\theta$.
\begin{assumption} \label{assum:support}
For each agent and all $\theta\in\boldsymbol{\Theta}$, the following properties hold:
\begin{itemize}
    \item The distributions $Q$ and $P(\omega|\boldsymbol{\phi}_\theta)$ are bounded, i.e., $Q<\infty$ and $P(\omega|\boldsymbol{\phi}_\theta)<\infty$ for all values of $\omega$, and
    \item The distribution $Q$ is absolutely continuous\footnote{A measure $\mu$ is dominated by (or absolutely continuous with respect to) a measure $\lambda$ if $\lambda(B) = 0$ implies $\mu(B)=0$ for every measurable set~$B$.} with respect to $P(\omega|\boldsymbol{\phi}_\theta)$.
\end{itemize} 
\end{assumption}

Assumption~\ref{assum:support} ensures that both distributions $Q$ and $P(\cdot|\boldsymbol{\phi}_\theta)$ for all $\theta\in\boldsymbol{\Theta}$ do not include any singularities and that the support of $Q$ is within the support of $P(\cdot|\boldsymbol{\phi}_\theta)$. 

\begin{assumption}\label{assum:mle_kl}
There exists a unique set of parameters $\boldsymbol{\phi}_{\theta^*}$ and $\boldsymbol{\phi}_{\theta}$ for each agent and $\theta\in\boldsymbol{\Theta}$ such that 
\begin{align} \label{eq:phi_star}
    \boldsymbol{\phi}_{\theta^*} &= \argmin_{\boldsymbol{\phi}} D_{KL}(Q\|P(\cdot|\boldsymbol{\phi})), \ \text{and} \nonumber \\
    \boldsymbol{\phi}_{\theta} &= \argmin_{\boldsymbol{\phi}} D_{KL}(Q_\theta\|P(\cdot|\boldsymbol{\phi})),
\end{align}
where $\phi_\theta$ could equal $\phi_{\theta^*}$. 
\end{assumption}
Assumption~\ref{assum:mle_kl} means that the set of parameters that minimize the KL divergence between the ground truth and the likelihood models for $\theta^*$ is unique. Furthermore, the set of parameters that minimize the KL between the underlying models $Q_\theta$ and the likelihood model for $\theta$ is unique.\footnote{Note that when the KL divergence is greater than $0$, the likelihood models are misspecified and model uncertainty is present.}

\begin{assumption}\label{assum:reg}
For each agent and all $\theta\in\boldsymbol{\Theta}$, the likelihood function $P(\cdot|\boldsymbol{\phi}_\theta)$ and the conjugate prior $f_0(\boldsymbol{\phi})$ abide by the regularity conditions as stipulated in~\cite{W1969,KV2012}. 
\end{assumption}

Assumption~\ref{assum:reg} holds when the distributions are sufficiently smooth and Assumption~\ref{assum:mle_kl} hold. For a detailed list of the regularity conditions, please see Appendix~\ref{app:regularity}. 

Next, we make the following proposition regarding the asymptotic properties of the conjugate distributions. 

\begin{proposition} \label{lem:normal_posterior}
        Let Assumptions~\ref{assum:support}-\ref{assum:reg} hold. Then, the posterior distribution of $\boldsymbol{\phi}$ has the following properties: \begin{enumerate}[label=(\alph*)]
            \item When prior evidence is finite, i.e., $|\mathbf{r}_\theta|<\infty$, and number of observations grows unboundedly, i.e., $t\to\infty$, the posterior distribution $f(\boldsymbol{\phi}|\psi(\{\boldsymbol{\omega}_{1:t},\mathbf{r}_\theta \}))$ converges in probability as 
            \begin{align} \label{eq:con1}
                &\lim_{t\to\infty} f(\boldsymbol{\phi}|\psi(\{\boldsymbol{\omega}_{1:t},\mathbf{r}_\theta \})) =\delta_{\boldsymbol{\phi}_{\theta^*}}(\boldsymbol{\phi}) 
            \end{align}
            where $\delta_{\boldsymbol{\phi}_\theta}$ is the Kronecker delta function centered at $\boldsymbol{\phi}_\theta$ and $\boldsymbol{\phi}_{\theta^*}$ is defined in~\eqref{eq:phi_star}. 
            \item When an agent is certain, i.e., $|\mathbf{r}_\theta|=\infty$, the posterior distribution $f(\boldsymbol{\phi}|\psi(\mathbf{r}_\theta))$ computed before collecting observations converges in probability as
            \begin{align} \label{eq:con2}
                \lim_{|\mathbf{r}_\theta|\to\infty} f(\boldsymbol{\phi}|\psi(\mathbf{r}_\theta))= \delta_{\boldsymbol{\phi}_\theta}(\boldsymbol{\phi}).
            \end{align}
            where $\boldsymbol{\phi}_\theta$ is defined in~\eqref{eq:phi_star}. 
        \end{enumerate}
\end{proposition}

Proposition~\ref{lem:normal_posterior} follows from the Bernstein-Von Mises theorem \cite{W1969, KV2012, LeCam1956, gelman2013, D1976,Ait1975, Kom1996, CG1999}, where the posterior distributions are asymptotically normal with a covariance decaying to $0$ asymptotically. Moreover, the mean is the maximum likelihood estimate that minimizes the KL divergence. This indicates that there exist families of parametric distributions that allow for the posterior (or prior) predictive distribution to asymptotically converge to the likelihood function evaluated at the exact parameters. Now, we are ready to analyze the asymptotic properties of the uncertain likelihood ration in~\eqref{eq:ULR}.

\begin{theorem} \label{thm:ULR}
Let Assumptions~\ref{assum:support}-\ref{assum:reg} hold, and assume that the amount of prior evidence for each hypothesis $\theta\in\boldsymbol{\Theta}$ is finite, i.e., $|\mathbf{r}_\theta|<\infty$. Then, the uncertain likelihood ratio~\eqref{eq:ULR}. converges in probability as 
\begin{align} \label{eq:tilde_ULR}
    \widetilde{\Lambda}_{\theta} = \lim_{t\to\infty}\Lambda_\theta(t)= 
    \frac{f(\boldsymbol{\phi}_{\theta^*}|\psi(\mathbf{r}_\theta))}{f_0(\boldsymbol{\phi}_{\theta^*})}.
\end{align}

\end{theorem}

\begin{proof}
First, the uncertain likelihood ratio is written as 
\begin{align} \label{eq:proof:ulr}
    \Lambda_{\theta}(t) &= \frac{\int_{\boldsymbol{\Phi}}P(\boldsymbol{\omega}_{1:t}|\boldsymbol{\phi})P(\mathbf{r}_\theta|\boldsymbol{\phi})f_0(\boldsymbol{\phi})d\boldsymbol{\phi}}{\int_{\boldsymbol{\Phi}}P(\boldsymbol{\omega}_{1:t}|\boldsymbol{\phi})f_0(\boldsymbol{\phi})d\boldsymbol{\phi} \int_{\boldsymbol{\Phi}}P(\mathbf{r}_\theta|\boldsymbol{\phi})f_0(\boldsymbol{\phi})d\boldsymbol{\phi}} \nonumber \\
    &= \frac{\int_{\boldsymbol{\Phi}}P(\mathbf{r}_\theta|\boldsymbol{\phi})f(\boldsymbol{\phi}|\psi(\boldsymbol{\omega}_{1:t}))d\boldsymbol{\phi}}{\int_{\boldsymbol{\Phi}}P(\mathbf{r}_\theta|\boldsymbol{\phi})f_0(\boldsymbol{\phi})d\boldsymbol{\phi}},
\end{align}
by simply applying Bayes' rule to update the posterior distribution $f(\boldsymbol{\phi}|\psi(\boldsymbol{\omega}_{1:t}))$. Then, using property (a) of Proposition~\ref{lem:normal_posterior}, the posterior distribution converges in probability to a delta function as $t\to\infty$, i.e., $\lim_{t\to\infty} f(\boldsymbol{\phi}|\psi(\boldsymbol{\omega}_{1:t}))=\delta_{\boldsymbol{\phi}_{\theta^*}}(\boldsymbol{\phi})$. This results in~\eqref{eq:proof:ulr} converging in probability to 
\begin{align}\label{eq:asymptotic_L}
    \Lambda_{\theta}(t) = \frac{P(\mathbf{r}_\theta|\boldsymbol{\phi}_{\theta^*})}{\hat{P}(\mathbf{r}_\theta)},
\end{align}
where $\hat{P}(\mathbf{r}_\theta)=\int_{\boldsymbol{\Phi}}P(\mathbf{r}_\theta|\boldsymbol{\phi})f_0(\boldsymbol{\phi})d\boldsymbol{\phi}$. Then, if we multiply the right hand side of~\eqref{eq:asymptotic_L} by $f_0(\boldsymbol{\phi}_{\theta^*})/f_0(\boldsymbol{\phi}_{\theta^*})$ and apply Bayes rule (see~\eqref{eq:cp}),~\eqref{eq:tilde_ULR} follows directly. 

\end{proof}

Following the same logic as in Theorem~\ref{thm:ULR}, we now provide a consistency result.

\begin{corollary} \label{cor:ULR_Dog}
Let Assumptions~\ref{assum:support}-~\ref{assum:reg} hold and assume that the agent is certain for each hypothesis $\theta\in\boldsymbol{\Theta}$, i.e., $|\mathbf{r}_\theta|{=}\infty$. Then, the uncertain likelihood ratio \eqref{eq:ULR} has the following property:
\begin{align}
&\lim_{t\to\infty,|\mathbf{r}_\theta| \to\infty} \Lambda_{\theta}(t) {=} \infty, \ \text{if $\boldsymbol{\phi}_\theta = \boldsymbol{\phi}_{\theta^*}$, and} \nonumber \\
&\lim_{t\to\infty, |\mathbf{r}_\theta| \to\infty} \Lambda_{\theta}(t) = 0, \ \text{if $\boldsymbol{\phi}_\theta \ne \boldsymbol{\phi}_{\theta^*}$.}
\end{align}
\end{corollary}

\begin{proof}
First, noting that the prior evidence $\mathbf{r}_\theta$ and private observations $\boldsymbol{\omega}_{1:t}$ are i.i.d, the order in which the data is received is also independent, allowing the limiting operations to be interchanged such that Theorem~\ref{thm:ULR} can be leveraged. Thus, 
\begin{align} \label{eq:interate_limit}
    \lim_{\tiny\begin{array}{c}
|\mathbf{r}_\theta|\to\infty\\t\to\infty
\end{array}} \Lambda_\theta(t) = \lim_{|\mathbf{r}_\theta|\to\infty} \widetilde{\Lambda}_\theta = \lim_{|\mathbf{r}_\theta|\to\infty} \frac{f(\boldsymbol{\phi}_{\theta^*}|\psi(\mathbf{r}_\theta))}{f_0(\boldsymbol{\phi}_{\theta^*})}. 
\end{align}
Noting from Proposition~\ref{lem:normal_posterior} that $f(\boldsymbol{\phi}_{\theta^*}|\psi(\mathbf{r}_\theta))$ converges in probability to $\delta_{\boldsymbol{\phi}_\theta}(\boldsymbol{\phi})$, the numerator in~\eqref{eq:interate_limit} either diverges to $\infty$ if $\boldsymbol{\phi}_{\theta^*}=\boldsymbol{\phi}_\theta$ or converges to $0$ otherwise. Then, our result is achieved since $f_0(\boldsymbol{\phi}_{\theta^*})$ to be strictly positive and bounded. Note that this result is also achieved when the limit as $|\mathbf{r}_\theta|\to\infty$ is applied first, thus justifying the interchange of operations. 
\end{proof}

Theorem~\ref{thm:ULR} shows that the uncertain likelihood ratio eventually converges to a finite value when the amount of prior evidence is finite, and the number of observations grows without bound. When the ground truth parameters $\boldsymbol{\phi}_{\theta^*}$ lie near the mode of the distribution $f(\boldsymbol{\phi}|\psi(\mathbf{r}_\theta))$, the ratio will converge to a value greater than $1$, indicating that the prior evidence and the observations are consistent. However, when the parameters $\boldsymbol{\phi}_{\theta^*}$ lie within the tail of $f(\boldsymbol{\phi}|\psi(\mathbf{r}_\theta))$, the ratio will converge to a value much less than $1$, indicating that the distributions are inconsistent with each other. Then, Corollary~\ref{cor:ULR_Dog} shows that as the amount of prior evidence grows without bound, only the hypothesis with parameters $\boldsymbol{\phi}_\theta =\boldsymbol{\phi}_{\theta^*}= \argmin_{\boldsymbol{\phi}\in\boldsymbol{\Phi}} D_{KL}(Q\|P(\cdot|\boldsymbol{\phi}))$ will have a ratio $>0$, allowing the agents to learn the set of hypotheses indistinguishable with the ground truth, i.e., $\boldsymbol{\Theta}^*$. These results are consistent with but more general than the uncertain models presented in~\cite{HUKJ2020_TSP,HUKJ2020_Gauss}. 

\section{Non-Bayesian Social Learning with Uncertain Models} \label{sec:nbsl_um}

In this section, we switch back to the networked setting and present a belief update rule that adjusts~\eqref{eq:main} to handle both epistemic and model uncertainties (c.f. Section~\ref{sec:types_of_uncertainty}). 

Consider that each agent $i$ collects prior evidence $\mathbf{r}_\theta^i$ and constructs a belief $\mu_{0}^i(\theta)=1$ for each hypothesis $\theta \in\boldsymbol{\Theta}$ at time $t=0$.Then, for each time step $t\ge1$, each agent sequentially (i) communicates their beliefs to their neighbors, (ii) receives a new observation, and (iii) updates their beliefs using a social learning rule. This results in each agent $i$ having access to the information $\boldsymbol{\gamma}_{t+1}^i(\theta) = \{\omega_{t+1}^i, \mathbf{r}_\theta^i, \{\mu_{t}^j(\theta)\}_{\forall j\in\mathcal{M}^i}\}$ for each hypothesis $\theta\in\boldsymbol{\Theta}$ at each time step $t+1$. Then, agent $i$ updates their belief $\mu_{t+1}^i(\theta)$ using the following update rule:
\begin{eqnarray} \label{eq:ll_update}
\mu_{t+1}^i(\theta) = \ell_{\theta}^i(\omega_{t+1}^i)\prod_{j\in\mathcal{M}^i} \mu_{t}^j(\theta)^{[\mathbf{A}]_{ij}}, \label{eq:LL_rule} 
\end{eqnarray}
where the product on the right hand side of (\ref{eq:LL_rule}) represents a geometric average of their neighbors beliefs and $\ell_{\theta}^i(\omega_{t+1}^i)$ is the uncertain likelihood update defined in~\eqref{eq:ulu}. Note that the social learning rule~\eqref{eq:LL_rule} differs from~\eqref{eq:main}.

With the proposed social learning rule stated, we now present the convergence properties of the beliefs with uncertain models. Note that auxiliary lemmas and proofs are provided in the appendix for ease of presentation. 

\begin{theorem} \label{th:LL}
Let Assumptions~\ref{assum:graph} and~\ref{assum:support}-\ref{assum:reg} hold and assume that the amount of prior evidence for each agent $i$ and each hypothesis $\theta\in \boldsymbol{\Theta}$ is finite, i.e., $|\mathbf{r}_\theta^i|<\infty$. Then, the beliefs generated with the update rule~\eqref{eq:ll_update} have the following property;
\begin{equation}
\lim_{t\to\infty} \mu_{t}^i(\theta) = \left(\prod_{j=1}^m \widetilde{\Lambda}_{\theta}^j\right)^{\frac{1}{m}} 
\end{equation}
in probability for each $i\in\mathcal{M}$, where $\widetilde{\Lambda}_{\theta}^j$ is defined in~\eqref{eq:tilde_ULR}.
\end{theorem}

Theorem~\ref{th:LL} states that the beliefs reach consensus and converge to the geometric average of the agents' asymptotic uncertain likelihood ratios. In fact asymptotically, the beliefs have a one-to-one relation with the centralized uncertain likelihood ratio defined as follows 
\begin{align}
    \lim_{t\to\infty} \frac{\mathbb{P}_\theta(\boldsymbol{\omega}_{1:t}^1,...,\boldsymbol{\omega}_{1:t}^m|\mathbf{r}_\theta^1,...,\mathbf{r}_\theta^m)}{\mathbb{P}_\theta(\boldsymbol{\omega}_{1:t}^1,...,\boldsymbol{\omega}_{1:t}^m|\emptyset,...,\emptyset)} = \prod_{i=1}^m \widetilde{\Lambda}_\theta^i.
\end{align}
The centralized solution follows from the fact that the observations of each agent are i.i.d. and independent over the network. 

Since each of agent $i$'s uncertain likelihood ratios converge to a finite value when $|\mathbf{r}_\theta|<\infty$, $\forall \theta\in\boldsymbol{\Theta}$, every agents' beliefs will converge to a value between $(0,\infty)$ and are interpreted in the same fashion as discussed in Section~\ref{sec:gum} after Corollary~\ref{cor:ULR_Dog}. Informally, a much greater belief than $1$ provides evidence that the prior evidence and observations are consistent, i.e., drawn from the same distribution, while a value much less than $1$ provides evidence that the data sets are inconsistent. Just as before, the agents cannot decide on the hypothesis that exactly matches the ground truth with finite evidence. However, when all of the agents become certain, i.e., $|\mathbf{r}_\theta^i|\to\infty$ $\forall i\in\mathcal{M}, \theta\in\boldsymbol{\Theta}$, learning the ground truth hypothesis is possible, as captured by the following theorem.

\begin{theorem} \label{cor:dogmatic}
Let Assumptions~\ref{assum:unique}-\ref{assum:reg} 
hold and assume that every agents is certain, i.e., $|\mathbf{r}_\theta^i|\to\infty$ $\forall i\in\mathcal{M}$. Then, the beliefs for each agent $i\in\mathcal{M}$ generated using the update rule \eqref{eq:LL_rule} have the following property: 
\begin{align}
    \lim_{t\to\infty, |\mathbf{r}_\theta^i|\to\infty} \mu_{t}^i(\theta) = \left\{ 
    \begin{array}{ll}
       \infty  & \text{if} \ \theta=\theta^*, \\
        0 & \text{otherwise}
    \end{array} 
    \right.
\end{align}
in probability. 

\end{theorem}

Theorem~\ref{cor:dogmatic} indicates that the collective group of agents can uniquely identify the ground truth hypothesis since the agents' beliefs for every $\theta\ne\theta^*$ will converge to $0$ and only the hypothesis $\theta=\theta^*$ will diverge to $\infty$. This result is consistent with traditional non-Bayesian social learning theory, indicating that general uncertain models allow agents to learn asymptotically.

\section{Implementation of Uncertain Models} \label{sec:implement}
This section provides the computational steps required to determine the uncertain models for implementation. Once the steps are formalized, the policies for multinomial and Multivariate Gaussian likelihood models are presented to provide examples of discrete and continuous uncertain models.

\subsection{Steps to identify the Uncertain Likelihood Update}

\textbf{Step 1: Select  n a family of probability distributions.} The first step requires identifying a reasonable family of parameterized distributions suitable for the particular application, i.e., $P^i(\cdot|\boldsymbol{\phi})$. First, the sensing device dictates whether a discrete or continuous distribution family describes the measurements. For instance, photon counts of an infrared detector are Poisson distributed, while pressure measurements of an acoustic sensor lead to a continuous distribution. The underlying physics of the phenomenon observed provides insights into a proper family of parametric distributions to model the measurements for different world states. One should be conscious of balancing the dimensionality of the parameters with the ability to distinguish between the various world states. Furthermore, note that the family of the likelihood model must abide by Assumption~\ref{assum:reg}.

\textbf{Step 2: Determine the natural conjugate prior.} Next, it is well known that when the likelihood model has a fixed-dimensional set of sufficient statistics, there exists a conjugate prior distribution \cite{D2005}, $f_0(\boldsymbol{\phi}^i)=f(\boldsymbol{\phi}^i|\psi(\emptyset))$, where $\psi(\emptyset)=\psi_0^i$ are the natural hyperparameters and are chosen such that the prior is uninformative. The conjugate prior is chosen based on the parameters that are unknown to the agent. For example, if the family of likelihoods is Gaussian with unknown mean and known variance, the conjugate prior is also a Gaussian distribution. A detailed list/discussion of conjugate prior distributions can be found in~\cite{D2005, F1997}.

\textbf{Step 3: Determine how to update the hyperparameters.} The benefit of using conjugate priors is that they are closed under multiplication, which enables the posterior distribution $f(\boldsymbol{\phi}^i|\psi^i(\mathbf{x}))$ computed based on a set of data (observations and/or prior evidence) $\mathbf{x}$ to be in the same family of distributions as $f_0(\boldsymbol{\phi}^i)$, with hyperparameters that are updated in a simple recursive form. Consider that the agent has previously collected data points $\mathbf{x}$, which lead to the hyperparameters $\psi^i(\mathbf{x})$. Next, the agent collects a new set of measurements $\mathbf{x}^+$. Then, there exists a function $g(\mathbf{x}^+,\psi^i(\mathbf{x}))$ s.t. 
\begin{align}
    \psi^i(\{\mathbf{x}^+,\mathbf{x}\})=g(\mathbf{x}^+,\psi^i(\mathbf{x})).
\end{align}
 Initially, the agent collects a set of prior evidence $\mathbf{r}_\theta^i$, which leads to the prior hyperparameters $\psi^i(\mathbf{r}_\theta^i)=g(\mathbf{r}_\theta^i,\psi^i_0)$, where $\psi^i_0$ are the set of vacuous hyperparameters that provide a noninformative prior. Then, the agent sequentially collects observations $\omega_{t}^i$ at time $t\ge1$ and updates the hyperparameters of the uncertain likelihood~\eqref{eq:ul} according to $\psi^i(\{\mathbf{r}_\theta^i, \boldsymbol{\omega}_{1:t}^i\})=g(\omega_{t}^i, \psi^i(\{\mathbf{r}_\theta^i, \boldsymbol{\omega}_{1:t-1}^i\}))$. At the same time, the agent also keeps track of the hyperparameters of the model of complete ignorance (defined below~\eqref{eq:ULR}) according to $\psi^i(\boldsymbol{\omega}_{1:t}^i)=g(\omega_{t}^i,\psi^i(\boldsymbol{\omega}_{1:t-1}^i))$. Examples of $g(\mathbf{x}^+,\psi^i(\mathbf{x}))$ can be found in the following subsections. 

\textbf{Step 4: Compute the normalization factor $Z(\mathbf{x}^+,\mathbf{x})$.} When the agent collects a new data point $\mathbf{x}^+$ and has previously collected the data set $\mathbf{x}$, the posterior conjugate distribution is computed as follows 
\begin{align}\label{eq:post_phi}
    f(\boldsymbol{\phi}^i|\psi^i(\{\mathbf{x}^+,\mathbf{x}\})) = \frac{P^i(\mathbf{x}^+|\boldsymbol{\phi}^i)P^i(\mathbf{x}|\boldsymbol{\phi}^i)f_0(\boldsymbol{\phi}^i)}{Z(\mathbf{x}^+,\mathbf{x})}, 
\end{align}
where 
\begin{align}
    Z(\mathbf{x}^+,\mathbf{x})=\int_{\boldsymbol{\Phi}}P^i(\mathbf{x}^+|\boldsymbol{\phi}^i)P^i(\mathbf{x}|\boldsymbol{\phi})f_0(\boldsymbol{\phi}^i)d\boldsymbol{\phi}^i
\end{align}
is the normalization factor. In many situations, $Z(\mathbf{x}^+,\mathbf{x})$ can be analytically computed, see~\cite{D2005}. However, if $Z(\mathbf{x}^+,\mathbf{x})$ does not have a closed form, the agents must use numerical methods to compute the normalization factor, which significantly increases the runtime performance. 

\textbf{Step 5: Compute the uncertain likelihood update.} Next, the agent can simply compute the uncertain likelihood update~\eqref{eq:ulu} as a ratio of normalization factors. In this case, $\mathbf{x}^+$ is substituted with the measurement $\omega_{t+1}^i$. Then, if we expand the priors in~\eqref{eq:ulu}, as shown in~\eqref{eq:post_phi}, and replace the data set $\mathbf{x}$ with $\{\boldsymbol{\omega}_{1:t}^i,\mathbf{r}_\theta^i\}$ and $\boldsymbol{\omega}_{1:t}^i$ for the numerator and denominator, respectively, the uncertain likelihood update becomes, 
\begin{align} \label{eq:ell_z}
    \ell_{\theta}^i(\omega_{t+1}^i) = \frac{Z(\omega_{t+1}^i,\{\boldsymbol{\omega}_{1:t}^i,\mathbf{r}_\theta^i\})Z(\boldsymbol{\omega}_{1:t}^i)}{Z(\omega_{t+1}^i,\boldsymbol{\omega}_{1:t}^i)Z(\{\boldsymbol{\omega}_{1:t}^i,\mathbf{r}_\theta^i\})}.
\end{align}

\subsection{Multinomial Uncertain Models} \label{sec:mult_um}
The first example consists of the multinomial uncertain models presented in \cite{HUKJ2020_TSP}. Here, we assume that the prior evidence $\mathbf{r}_\theta^i=\{r_{1\theta}^i,...,r_{K\theta}^i\}$ for each hypothesis $\theta\in\boldsymbol{\Theta}$ represent a set of counts and each observation $\omega_{t}^i\in \{1,...,K\}$ is categorical where $K\ge 2$. This results in the agents assuming the family of multinomial distributions as their likelihood functions with parameters $\boldsymbol{\phi}_\theta^i=\boldsymbol{\pi}_\theta^i=\{\pi_{k\theta}^i| \forall k = 1,...,K\}$, where $\pi_{k\theta}^i$ is the probability that an observation is drawn from the $k$-th category and $\sum_{k=1}^K \pi_{k\theta}^i = 1$. The natural conjugate prior of the multinomial distribution is the Dirichlet distribution 
\begin{align} \label{eq:dir_prior}
    f(\boldsymbol{\pi}^i|\psi^i(\mathbf{x})) = \frac{1}{Z(\psi^i(\mathbf{x}))}\prod_{k=1}^K (\pi_{k}^i)^{\psi^i_{k}(\mathbf{x})-1},
\end{align}
where the normalization factor $Z(\psi^i(\mathbf{x}))=B(\psi^i(\mathbf{x})) = (\prod_{k=1}^K \Gamma(\psi_{k}^i(\mathbf{x})))/\Gamma(\sum_{k=1}^K \psi_{k}^i(\mathbf{x}))$ is the multivariate beta function. Initially, the hyperparameters $\psi_0^i=(\psi^i_{10},...,\psi^i_{K0})$ are vacuous and set to $\psi^i_{k0}=1$ for each category $k$, which applies a uniform distribution over the probability simplex of the parameters $\boldsymbol{\pi}^i$. 

Next, the agent computes the initial hyperparameters due to the prior evidence as $\psi^i(\mathbf{r}_\theta^i)=\mathbf{r}_\theta^i+\psi^i_0 =(r_{1\theta}^i+1,...,r_{K\theta}^i+1)$. Then, when the agent collects a new observation $\boldsymbol{\omega}_t^i=k$ at time $t\ge1$, the hyperparameters are updated as follows
\begin{align}
    \psi^i(\{\omega_t^i=k,\boldsymbol{x}\}) = \psi^i(\boldsymbol{x}) + \delta_k,
\end{align}
where $\delta_k$ is a vector of zeros with a 1 located in element $k$, i.e., the hyperparameters represent a set of counts of the combined data set $\{\omega_t^i,\boldsymbol{x}\}$. 

Then, utilizing the fact that 
\begin{align}
    \frac{Z(x^+=k,\mathbf{x})}{Z(\mathbf{x})}=\frac{B(\psi^i(\{x^+=k,\mathbf{x}\}))}{B(\psi^i(\mathbf{x}))}=\frac{\psi_k^i(\mathbf{x})+1}{|\mathbf{x}|+K},
\end{align}
the uncertain likelihood update~\eqref{eq:ell_z} becomes,
\begin{align} \label{eq:multinomial_ell}
    \ell_\theta^i(\omega_{t+1}^i=k) = \frac{(n_{kt}^i + r_{k\theta}^i+1)(t+K)}{(|\mathbf{r}_\theta^i|+t+K)(n_{kt}^i + 1)},
\end{align}
where $n_{kt}^i$ is the set of counts of category $k$ from data set $\boldsymbol{\omega}_{1:t}^i$. For a detailed derivation, please see~\cite{HUKJ2020_TSP}.

\subsection{Multivariate Gaussian Uncertain Models} \label{sec:mvn_models}
Next, we consider that the prior evidence and observations for agent $i$ are drawn from a $d$-dimensional Multivariate Gaussian distribution with unknown mean and variance, i.e., $\boldsymbol{\phi}^i=\{\mathbf{m}^i,\boldsymbol{\Sigma}^i\}$. The natural conjugate prior of the Multivariate Gaussian distribution is the Normal Inverse Wishart distribution~\cite[p.~133]{M2012} defined as
\begin{align} \label{eq:normal_cp}
    f(&\boldsymbol{\phi}^i|\psi^i(\mathbf{x})) =  \mathcal{N}\left(\mathbf{m}^i|\boldsymbol{\varpi}(\mathbf{x}),\frac{1}{\kappa(\mathbf{x})}\boldsymbol{\Sigma}^i\right)IW\left(\boldsymbol{\Sigma}^i|\mathbf{S}(\mathbf{x}),\nu(\mathbf{x})\right) \nonumber \\
    &= \frac{1}{Z(\psi^i(\mathbf{x}))} \Big|\boldsymbol{\Sigma}^i\Big|^{-\frac{\nu(\mathbf{x})+d+2}{2}} \exp\left(-\frac{1}{2}tr\left((\boldsymbol{\Sigma}^i)^{-1}\mathbf{S}(\mathbf{x})\right)\right)\times \nonumber \\
      & \ \ \ \exp\left( -\frac{\kappa(\mathbf{x})}{2}\left(\mathbf{m}^i-\boldsymbol{\varpi}(\mathbf{x})\right)'(\boldsymbol{\Sigma}^i)^{-1}\left(\mathbf{m}^i-\boldsymbol{\varpi}(\mathbf{x})\right)\right), 
\end{align}
where 
\begin{align} \label{eq:norm_gauss}
    Z(\psi^i(\mathbf{x})) = 2^{\frac{\nu(\mathbf{x}) d}{2}}&\Gamma_D\left(\frac{\nu(\mathbf{x})}{2}\right)\left(\frac{2\pi}{\kappa(\mathbf{x})}\right)^{\frac{d}{2}} \times \Big|\mathbf{S}(\mathbf{x})\Big|^{-\frac{\nu(\mathbf{x})}{2}} 
\end{align}
is the normalization factor and $\Gamma_d(\alpha)=\pi^{d(d-1)/4}\prod_{i=1}^d \Gamma((2\alpha + 1 - i)/2)$ is the multivariate gamma function. The initial hyperparameters are defined as $\psi_0^i=\{\boldsymbol{\varpi}_0, \kappa_0, \nu_0, \mathbf{S}_0\}$, where $\boldsymbol{\varpi}_0$ is the prior mean, $\kappa_0$ and $\nu_0$ are factors of how strongly we believe in the priors, and $\mathbf{S}_0$ is the prior mean for $\boldsymbol{\Sigma}^i$. A noniformative prior \cite{gelman2013} (suggested by Jeffery's) would suggest selecting the parameters $\boldsymbol{\varpi}_0=\mathbf{0}$, $\kappa_0\to 0$, $\nu_0\to-1$, and $|\mathbf{S}_0|\to 0$, which results in an improper prior. However in practice, it is suggested to use a weakly informative prior \cite{M2012}, where $\kappa_0$ is set to some small number, $\nu_0=d+2$, and $\boldsymbol{\varpi}_0$ and $\mathbf{S}_0$ are set based on some intuition of the data. In this work, we set $\kappa_0=1$, $\nu_0=d+2$, $\boldsymbol{\varpi}_0=\mathbf{0}$, and $\mathbf{S}_0=\mathbf{I}$ to be the hyperparameters of the model of complete ignorance, where $\mathbf{I}$ is the identify matrix.

\begin{figure*}[t!]
    \subfigure[2x2 Grid]{
        \centering
        \includegraphics[width=0.23\textwidth]{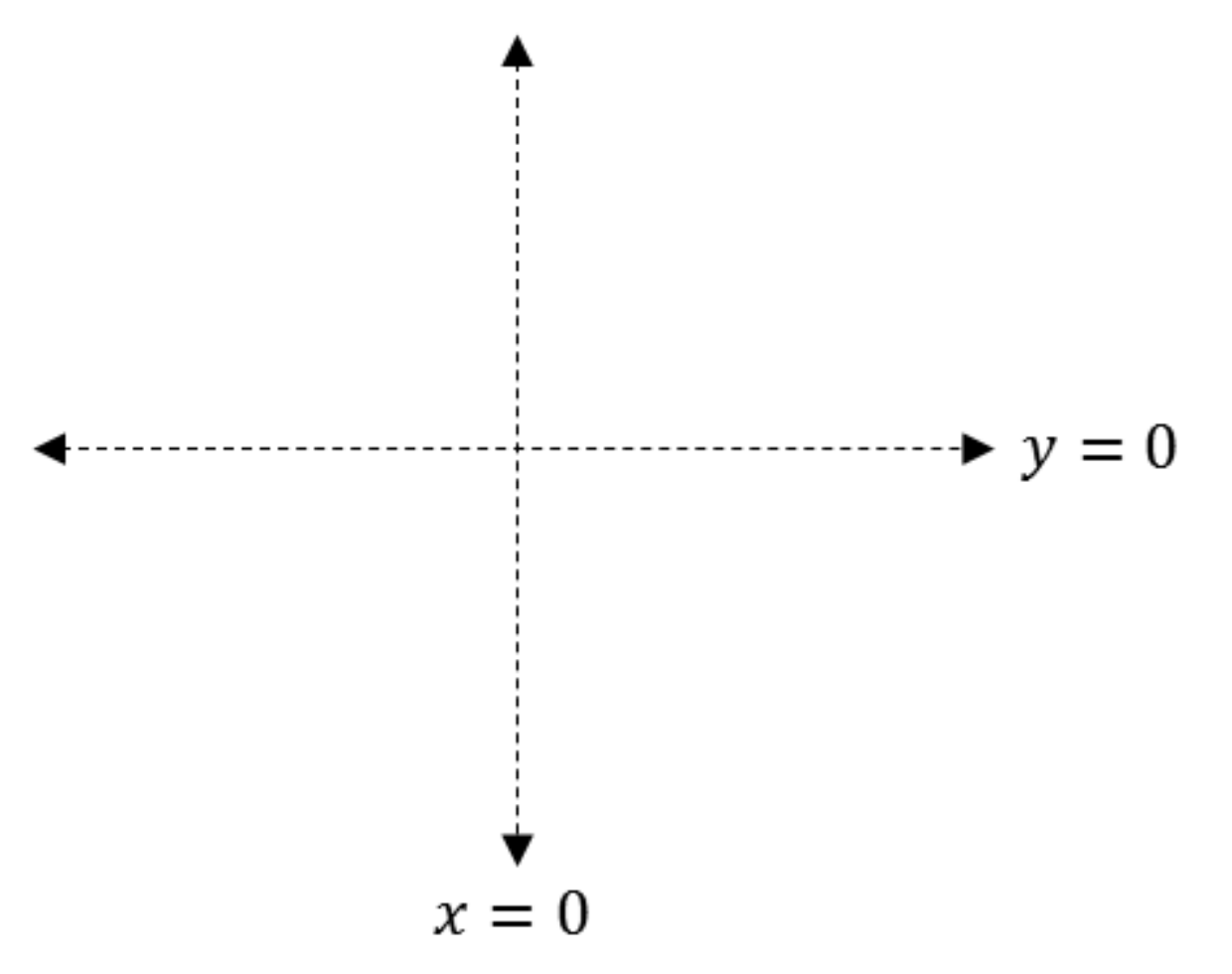}
    }%
    \subfigure[4x4 Grid]{
        \centering
        \includegraphics[width=0.23\textwidth]{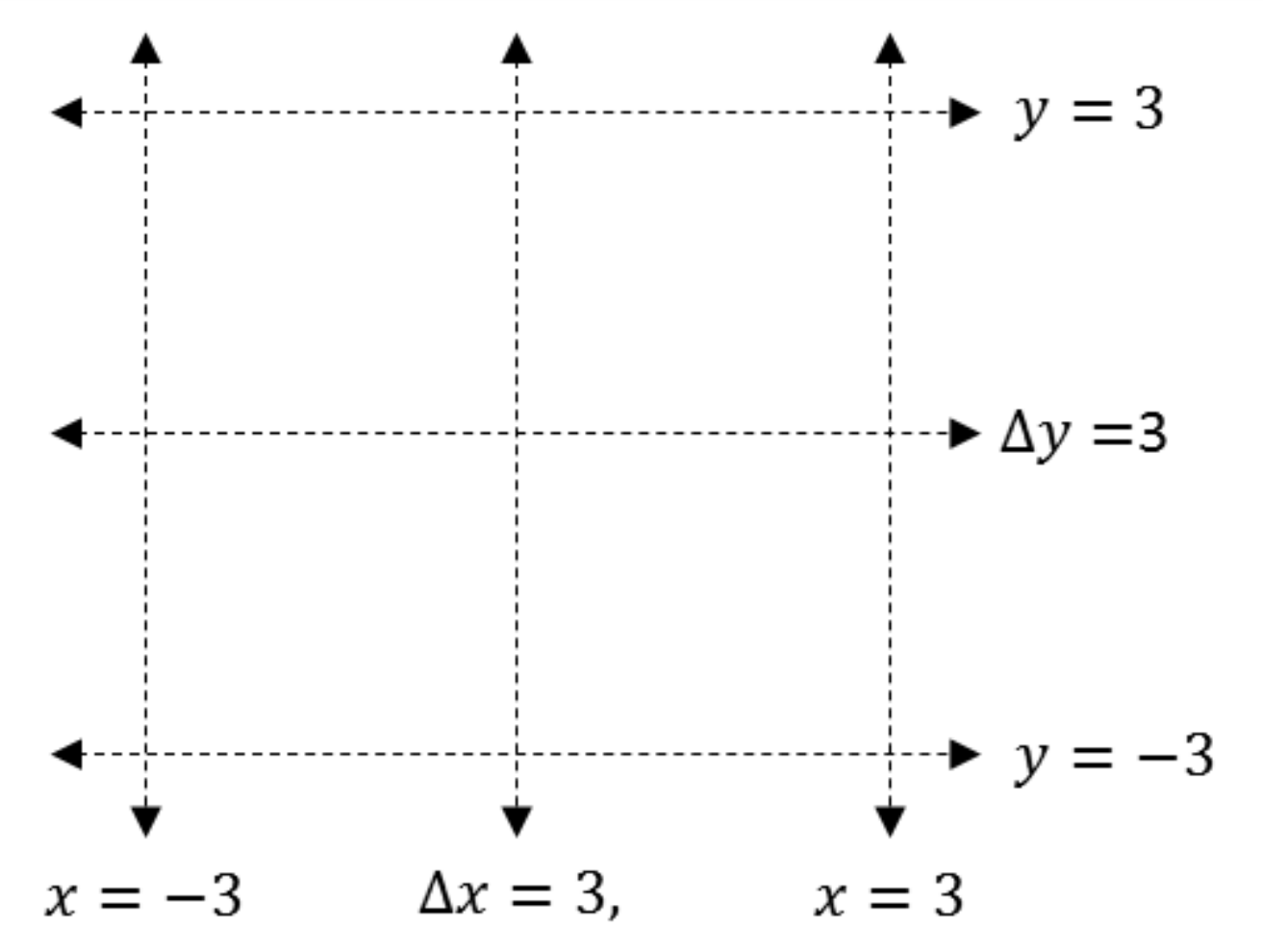}
    }%
    \subfigure[8x8 Grid]{
        \centering
        \includegraphics[width=0.23\textwidth]{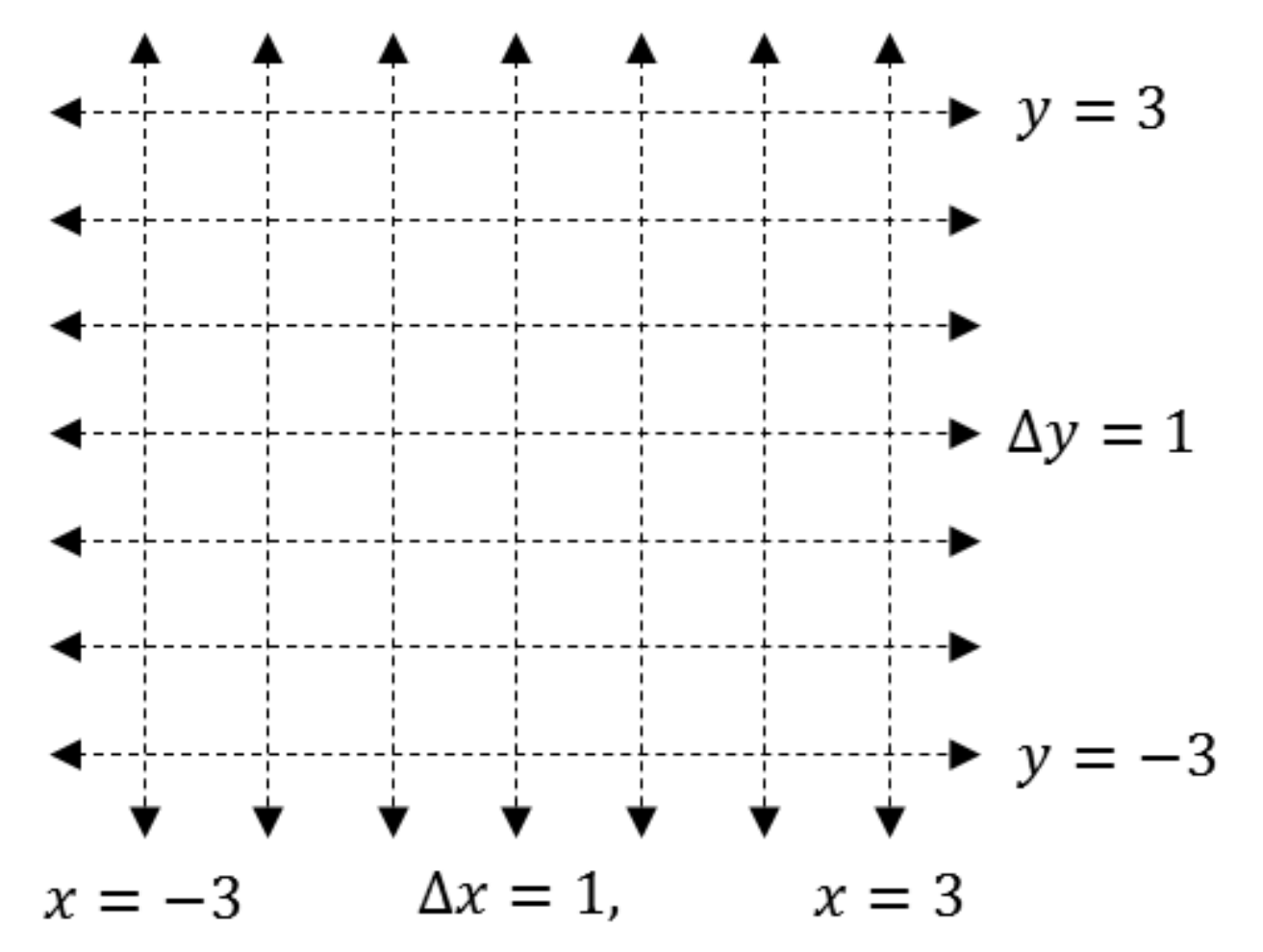}
    }%
    \subfigure[16x16 Grid]{
        \centering
        \includegraphics[width=0.23\textwidth]{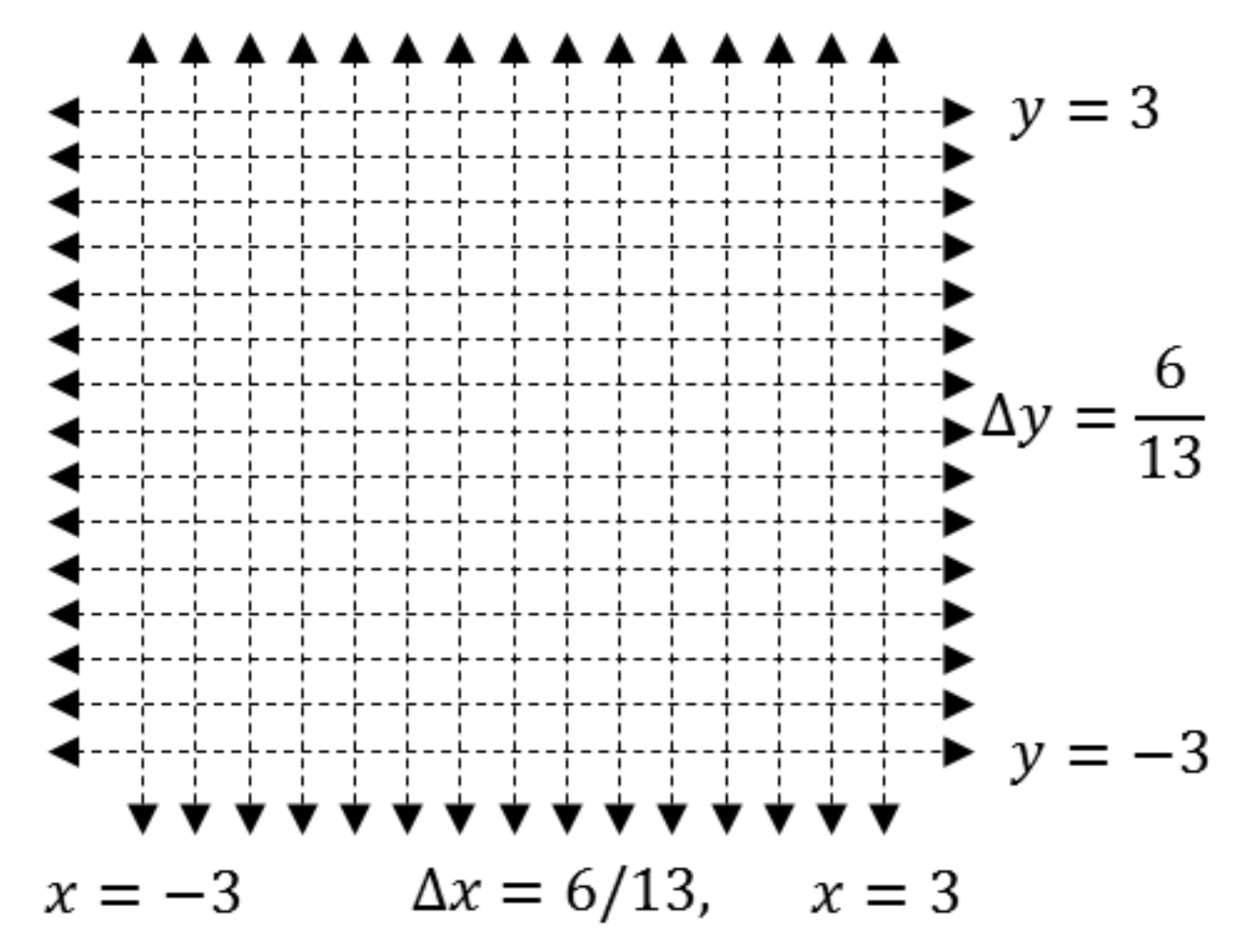}
    }

    \caption{Various partition structures tested} \vspace{-12pt}
    \label{fig:grid_partitions}
\end{figure*}

Next, the agent computes the initial hyperparameters due to the prior evidence as follows
\begin{align} \label{eq:niw_params}
    \kappa(\mathbf{r}_\theta^i)  & = \kappa_0 + |\mathbf{r}_\theta^i|; \ \ \
    \nu(\mathbf{r}_\theta^i) = \nu_0 + |\mathbf{r}_\theta^i|; \nonumber \\
    \boldsymbol{\varpi}(\mathbf{r}_\theta^i) & = \frac{\kappa_0 \boldsymbol{\varpi}_0 + |\mathbf{r}_\theta^i|\bar{\mathbf{r}}_\theta^i}{\kappa(\mathbf{r}_\theta^i)} \nonumber \\
    \mathbf{S}(\mathbf{r}_\theta^i) & = \mathbf{S}_0 + \bar{\mathbf{S}}(\mathbf{r}_\theta^i) + \kappa_0\mathbf{\varpi}_0\mathbf{\varpi}_0' - \kappa(\mathbf{r}_\theta^i) \boldsymbol{\varpi}(\mathbf{r}_\theta^i) \boldsymbol{\varpi}(\mathbf{r}_\theta^i)'
\end{align}
where $\bar{\mathbf{S}}(\mathbf{r}_\theta^i)=\sum_{k=1}^K r_{k\theta}^i (r_{k\theta}^i)'$ and $\bar{r}_\theta^i$ is the sample mean of the prior evidence.

Then, when the agent collects a new observation $\boldsymbol{\omega}_t^i$ at time $t\ge 1$, the hyperparameters are sequentially updated using the following recursive functions
\begin{align}
\kappa(\{\boldsymbol{\omega}_t^i, \mathbf{x}\}) & = \kappa(\mathbf{x})+1; \ \ \
    \nu(\{\boldsymbol{\omega}_t^i,\mathbf{x}\}) = \nu(\mathbf{x}) + 1; \nonumber \\
    \boldsymbol{\varpi}(\{\boldsymbol{\omega}_t^i,\mathbf{x}\}) & = \frac{\kappa(\mathbf{x}) \boldsymbol{\varpi}(\mathbf{x}) + \boldsymbol{\omega}_t^i}{\kappa(\{\boldsymbol{\omega}_t^i, \mathbf{x}\})}; \nonumber \\
    \mathbf{S}(\{\boldsymbol{\omega}_t^i,\mathbf{x}\}) & = \mathbf{S}(\mathbf{x}) + \boldsymbol{\omega}_t^i(\boldsymbol{\omega}_t^i)' + \kappa(\mathbf{x})\boldsymbol{\varpi}(\mathbf{x})\boldsymbol{\varpi}(\mathbf{x})' \nonumber \\ & \ \ \ - \kappa(\{\boldsymbol{\omega}_t^i,\mathbf{x}\})\boldsymbol{\varpi}(\{\boldsymbol{\omega}_t^i,\mathbf{x}\})\boldsymbol{\varpi}(\{\boldsymbol{\omega}_t^i,\mathbf{x}\})',
\end{align}
such that the data sets $\mathbf{x}$ are $\mathbf{x}=\{\boldsymbol{\omega}_{1:t-1}^i,\mathbf{r}_\theta^i\}$ and $\mathbf{x}=\boldsymbol{\omega}_{1:t-1}^i$, respectively. 

Next, we exploit the fact that a multivariate Student $t$-distribution can be written as a Gaussian mixture~\cite{M2012} resulting in the ratio of normalization factors~\eqref{eq:norm_gauss} taking the following form
\begin{align}
    \frac{Z(\mathbf{x}^+,\mathbf{x})}{Z(\mathbf{x})}= t_{\hat{\nu}(\mathbf{x})}(\mathbf{x}^+, \boldsymbol{\varpi}(\mathbf{x}), \hat{\mathbf{S}}(\mathbf{x})),
\end{align}
where $\hat{\nu}(\mathbf{x})=\nu(\mathbf{x})-d+1$ and
\begin{align}
    \hat{\mathbf{S}}(\mathbf{x}) = \frac{\kappa(\mathbf{x})+1}{\kappa(\mathbf{x})(\nu(\mathbf{x})-d+1)}\mathbf{S}(\mathbf{x}). \nonumber
\end{align}

Finally, the uncertain likelihood update is computed as,
\begin{align} \label{eq:mvn_ell}
    \ell_\theta^i(\boldsymbol{\omega}_{t+1}^i) = \frac{t_{\hat{\nu}(\{\boldsymbol{\omega}^i_{1:t},\mathbf{r}_\theta^i\})}(\boldsymbol{\omega}_{t+1}^i, \boldsymbol{\varpi}(\{\boldsymbol{\omega}^i_{1:t},\mathbf{r}_\theta^i\}), \hat{\mathbf{S}}(\{\boldsymbol{\omega}^i_{1:t},\mathbf{r}_\theta^i\}))}{t_{\hat{\nu}(\boldsymbol{\omega}^i_{1:t})}(\boldsymbol{\omega}_{t+1}^i, \boldsymbol{\varpi}^i(\boldsymbol{\omega}^i_{1:t}), \hat{\mathbf{S}}^i(\boldsymbol{\omega}^i_{1:t}))}.
\end{align}

\section{Nonparametric Framework for General Uncertain Models} \label{sec:example_non}

Thus far, we have restricted the parametric family of distributions $P^i(\cdot|\phi_\theta^i)$ for each agent $i$ to meet the regularity conditions of Assumption~\ref{assum:reg}. However, the best parametric family of distributions may not be known a priori.  While any distribution can be modeled by multimodal distributions, such as a mixture of Gaussians, the complexity of the resulting conjugate prior grows exponentially with more observations. Furthermore, it can be the case that the natural conjugate prior does not lead to an analytically computable normalization factor $Z(\mathbf{x}^+|\psi^i(\mathbf{x}))$. Therefore, this section presents a non-parametric approach to simplify the problem by modeling these challenging distributions as histograms with multinomial likelihoods.

\subsection{Modeling Uncertainty in Nonparametric Data using Multinomial Uncertain Models}
In this setting, we assume that each agent collects a set of prior evidence for each hypothesis, where each observation lies in a $d$-dimensional Euclidean space. Our goal is to partition the Euclidean space into a set of finite rectangular cuboids that allow prior evidence and future private observations to be mapped to a histogram. This then allows the agent to model the data as a multinomial distribution, which abides by the regularity conditions. 

Consider that each agent $i$ defines a rectangle $R^i=[c_1^1,c_{m_1}^1]\times ... \times [c_1^d, c_{m_d}^d]$ consisting of all $\mathbf{x}\in\mathbb{R}^d$ such that $c_1^h\le x^h< c_{m_h}^h$ for all $h=1,...,d$, which represents the space that contains majority of the prior evidence of every hypothesis. Then, let agent $i$ define a set of hyperplanes $\mathbf{c}_h = \{c_1^h, c_2^h,...,c_{m_d}^h\}$ for all $h=1,...,d$ such that the rectangle $R_i$ is partitioned into a rectilinear grid, which is a tessellation of the space into rectangular cuboids. We assume that the rectangular cuboids in $R^i$ are congruent. Now, consider that each hyperplane $l=1,...,m_h$ in dimension $h=1,...,d$, i.e., $c_l^h$, extends beyond the rectangle $R^i$ to partition the entire Euclidean space into a set of $K=\prod_{h=1}^d (m_h+1)$ rectangular cuboids. This results in rectangular cuboids outside of the rectangle $R_i$ that are not congruent with the inner cuboids. An example of a 2D Euclidean space partitioned into a rectilinear grid is shown in Figure~\ref{fig:grid_partitions}, where the black dotted lines represent the hyperplanes along each dimension.\footnote{In general, the dimension of the Euclidean space $d$, number of rectangular cuboids $K$, and location of hyperplanes $\mathbf{c}_h$ in each dimension $h$ could vary between agents. However, for ease of presentation, we assume in this work that they are the same for all agents.} 

\begin{figure*}[t!]
    \subfigure[$\theta_1=\theta^*$]{
        \centering
        \includegraphics[width=0.32\textwidth]{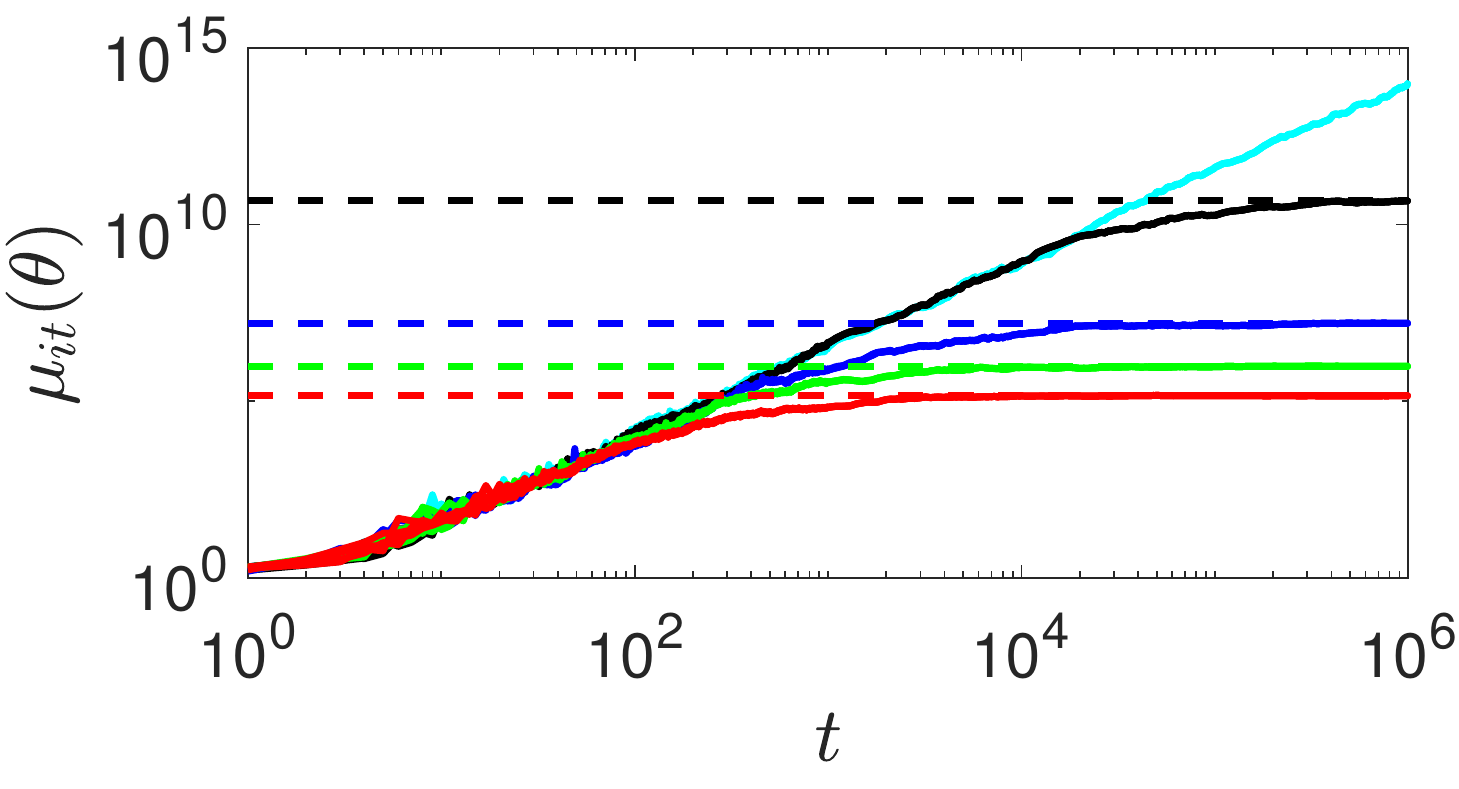}
    }%
    \subfigure[$\theta_2$]{
        \centering
        \includegraphics[width=0.32\textwidth]{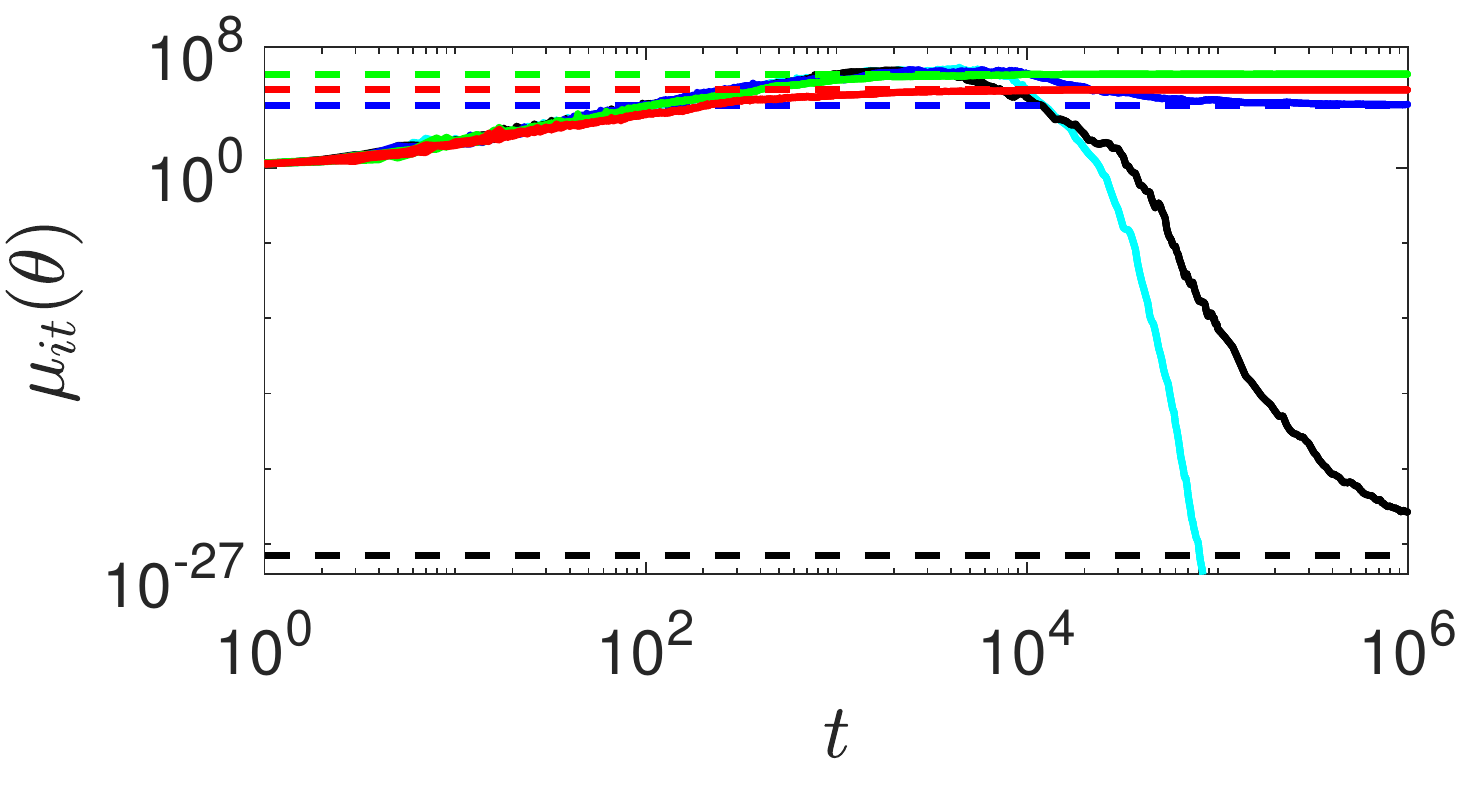}
    }%
    \subfigure[$\theta_3$]{
        \centering
        \includegraphics[width=0.32\textwidth]{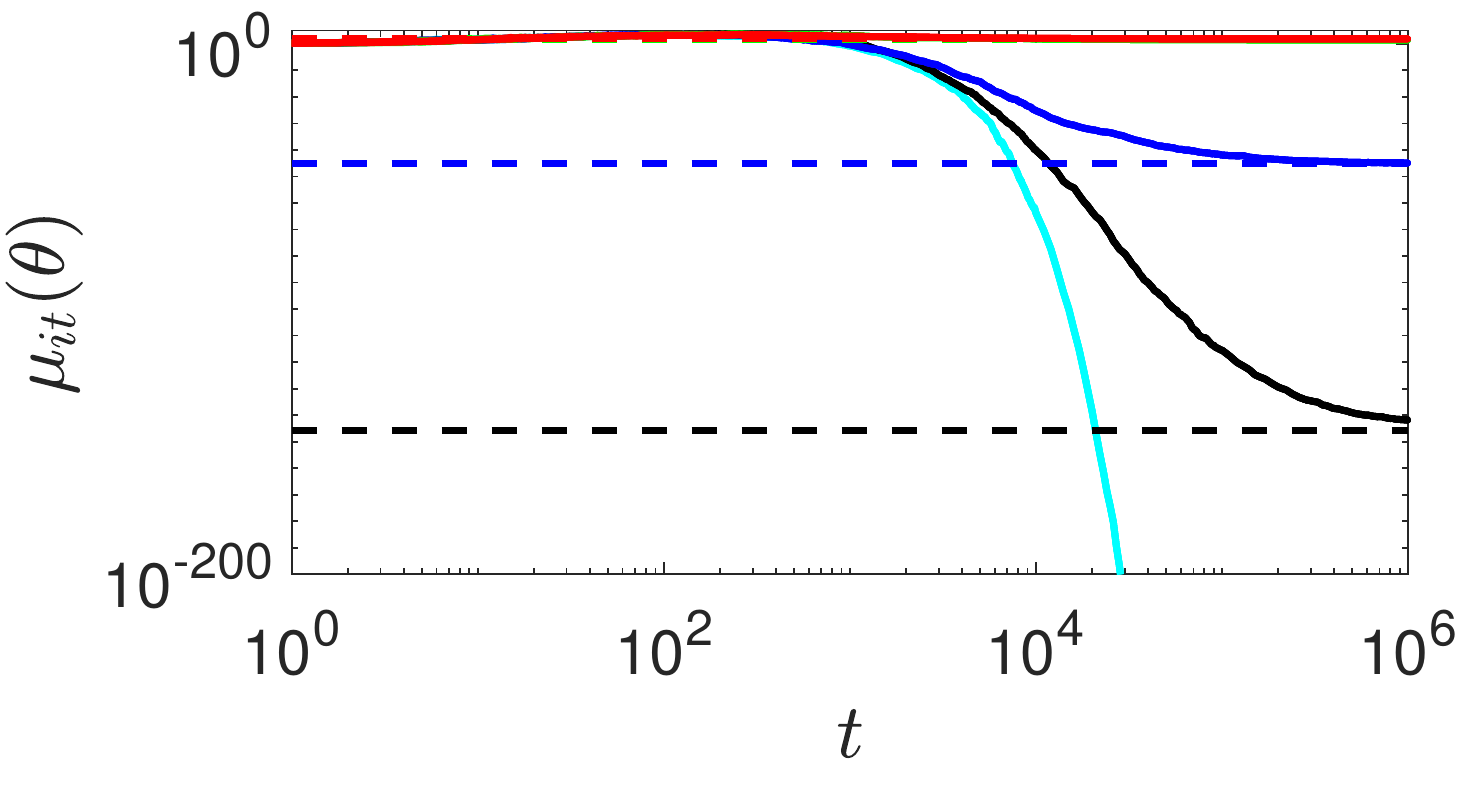}
    }

    \centering
    \subfigure{
        
        \includegraphics[width=0.75\textwidth]{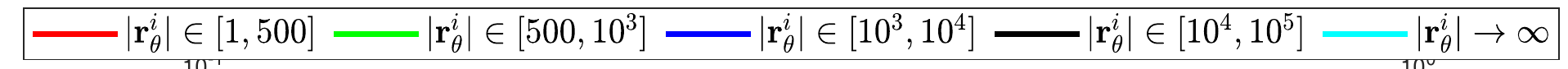}
    } \vspace{-6pt}
    \caption{Belief evolution of well-specified Gaussian uncertain models. The solid lines represent the beliefs $\mu_{t}^i(\theta)$, while the dashed lines represent the asymptotic point of convergence $(\prod_i^4 \widetilde{\Lambda}_{\theta}^i)^{(1/4)}$. } \vspace{-12pt}
    
    \label{fig:gauss_models_ws}
\end{figure*}

Once the Euclidean space is partitioned, the prior evidence for each hypothesis is mapped to a histogram where each bar represents the number of times the prior evidence falls within the specific rectangular cuboid. 
Now, we can represent the prior evidence as a vector of counts $\mathbf{r}_\theta^i=(r_{1\theta}^i, ..., r_{K\theta}^i)$, where $k=1,...,K$ is the index of the rectangular cuboid. This transformation allows the observations to be analogous to data being drawn from a multinomial distribution parameterized by a vector of probabilities $\boldsymbol{\pi}^i_\theta=(\pi_{1\theta}^i,...,\pi_{K\theta}^i)$. The posterior distribution of the parameters, i.e., $f(\boldsymbol{\pi}_\theta^i|\mathbf{r}_\theta^i)$ defined in~\eqref{eq:dir_prior}, are modeled according to a Dirichlet distribution, which abides by the regularity conditions, allowing our main results to hold \cite{HUKJ2020_TSP, UHLJ2019}. Therefore, the agents can implement the multinomial uncertain models, as presented in Section~\ref{sec:mult_um}. 

The main challenges associated with approximating a non-parametric distribution are as follows:
\begin{enumerate}
    \item How to design grids to distinguish the hypotheses, and what are the effects on the general beliefs as $K$ increases?
    \item Should the hyperplanes be uniformly spaced inside $\mathbb{R}^d$, or should they be selected based on the sample distribution of the prior evidence, e.g., a Voronoi cell? 
\end{enumerate}
In this preliminary study, we initiate the investigation of the first challenge above and leave a complete analysis of both challenges for future work.

\section{Numerical Analysis} \label{sec:results}
This section presents two numerical studies to validate the network's convergence properties with uncertain models empirically. First, we consider that the underlying ground truth distribution is Gaussian to verify that the results hold for well-specified likelihood models. Then, we consider that the ground truth distribution could be multimodal, non-parametric, or have a normalization factor that is hard to compute, to verify that the results hold for misspecified likelihood models. 

\begin{figure}[h!]
    \centering
    \includegraphics[width=0.5\columnwidth]{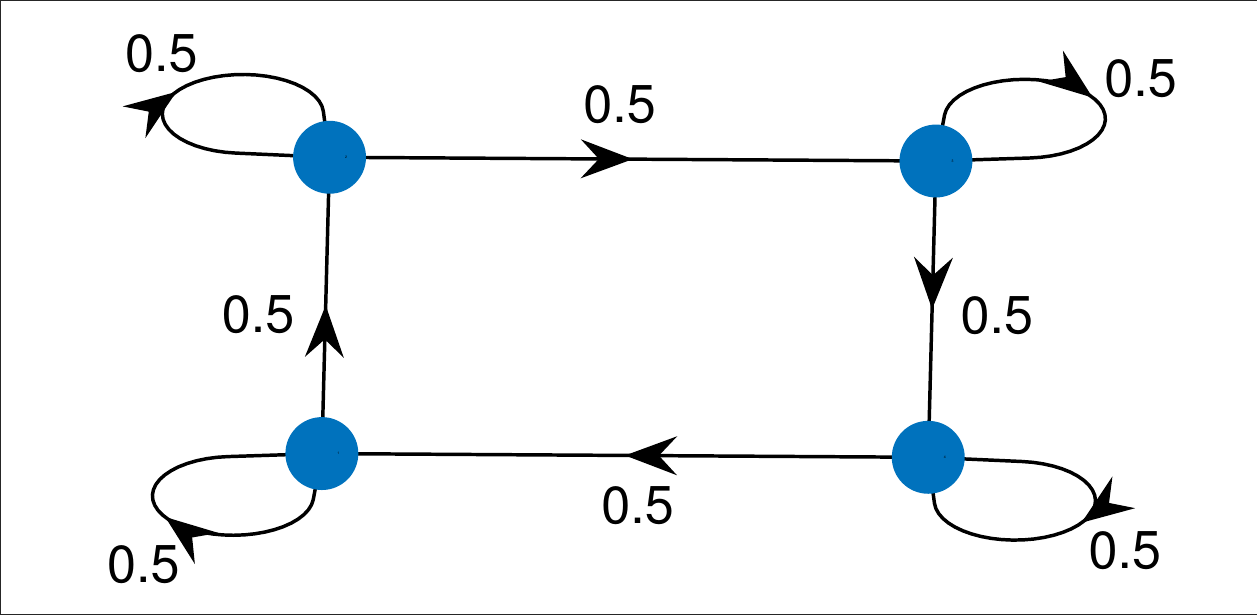}
    \caption{Social network structure considered in the numerical examples.} \label{fig:graph} \vspace{-12pt}
\end{figure}
\subsection{Well-specified likelihood models}
We begin by considering a network of 4 agents connected according to a directed cycle graph with self-loops, as seen in Figure~\ref{fig:graph}. In this example, we assume that the possible states of the world lead to three possible observation distributions, $Q_1$, $Q_2$, and $Q_3$. Each distribution is modeled as a multivariate Gaussian distribution, i.e., $Q_m \sim \mathcal{N}(\mathbf{m}_m,\boldsymbol{\Sigma}_m)$, where the mean and variance for each model is provided in Table~\ref{tab:gauss}. Furthermore, We assume that the observations for each agent $i$ are drawn from the first distribution, i.e., $\boldsymbol{\omega}_{1:t}^i\sim Q_1$. 

\begin{table}[]
    \centering 
    \caption{Definition of Multivariate Gaussian Parameters}
    \begin{tabular}{c|cc}
       $Q_1$  & $\mathbf{m}_1= \begin{bmatrix} 0 & 0 \end{bmatrix}$' &  $\boldsymbol{\Sigma}_1 = \begin{bmatrix} 1 & 0 \\ 0 & 1 \end{bmatrix}$ \\
       $Q_2$  & $\mathbf{m}_2= \begin{bmatrix} 0 & 0 \end{bmatrix}$' &  $\boldsymbol{\Sigma}_2 = \begin{bmatrix} 1.1 & 0 \\ 0 & 1.1 \end{bmatrix}$ \\ 
       $Q_3$  & $\mathbf{m}_3= \begin{bmatrix} 0 & 0 \end{bmatrix}$' & $\boldsymbol{\Sigma}_3 = \begin{bmatrix} 1.5 & 0 \\ 0 & 1.5 \end{bmatrix}$' \\
    \end{tabular} \vspace{-12pt}
    
    \label{tab:gauss}
\end{table}

Each agent is assumed to possess three hypotheses about the state of the world, i.e., $\boldsymbol{\Theta}=\{\theta_1, \theta_2, \theta_3\}$, such that $\theta_1=\theta^*$. The underlying distribution for each hypothesis seen by each agent $i$ is one of the three distributions given in Table~\ref{tab:hypotheses}. It is easy to see that social learning via the rule~\eqref{eq:LL_rule} is needed for the agents to collectively identify the true state of the world. 

\begin{table}[]
\caption{Definition of hypotheses for each agent}
\centering \label{tab:hypotheses}
\begin{tabular}{c|c|c|c|c}
& Agent 1 & Agent 2 & Agent 3 & Agent 4 \\ \hline
$\theta_1$ & $Q_1$  & $Q_1$ & $Q_1$ & $Q_1$ \\
$\theta_2$ & $Q_1$ & $Q_2$ & $Q_1$ & $Q_1$ \\
$\theta_3$ & $Q_1$  & $Q_1$ & $Q_3$ & $Q_1$ \\
\end{tabular} \vspace{-12pt}

\end{table}

We conducted five experiments, where each simulation consists of each agent collecting a random amount of prior evidence for each hypothesis drawn from a uniform random variable within the following ranges, $[1,500]$, $[500,10^3]$, $[10^3,10^4]$, $[10^4,10^5]$, and $\infty$, respectively. The first $4$ ranges are used to validate the results of Theorem~\ref{th:LL} numerically and present the effects of how the amount of prior evidence changes the overall point of convergence, while the final experiment is used to validate the results of Theorem~\ref{cor:dogmatic}.

We conducted 10 Monte Carlo simulation for each experiment, where the observations $\boldsymbol{\omega}_{1:t}^i$ were regenerated during each run. Here, we assume that the agents' likelihood models are well-specified, i.e., the likelihoods are Gaussian. Therefore, the agents implement the uncertain likelihood update~\eqref{eq:mvn_ell}, presented in Section~\ref{sec:mvn_models}.
\begin{figure*}[t!]
    \subfigure[$\theta_1=\theta^*$]{
        \centering
        \includegraphics[width=0.32\textwidth]{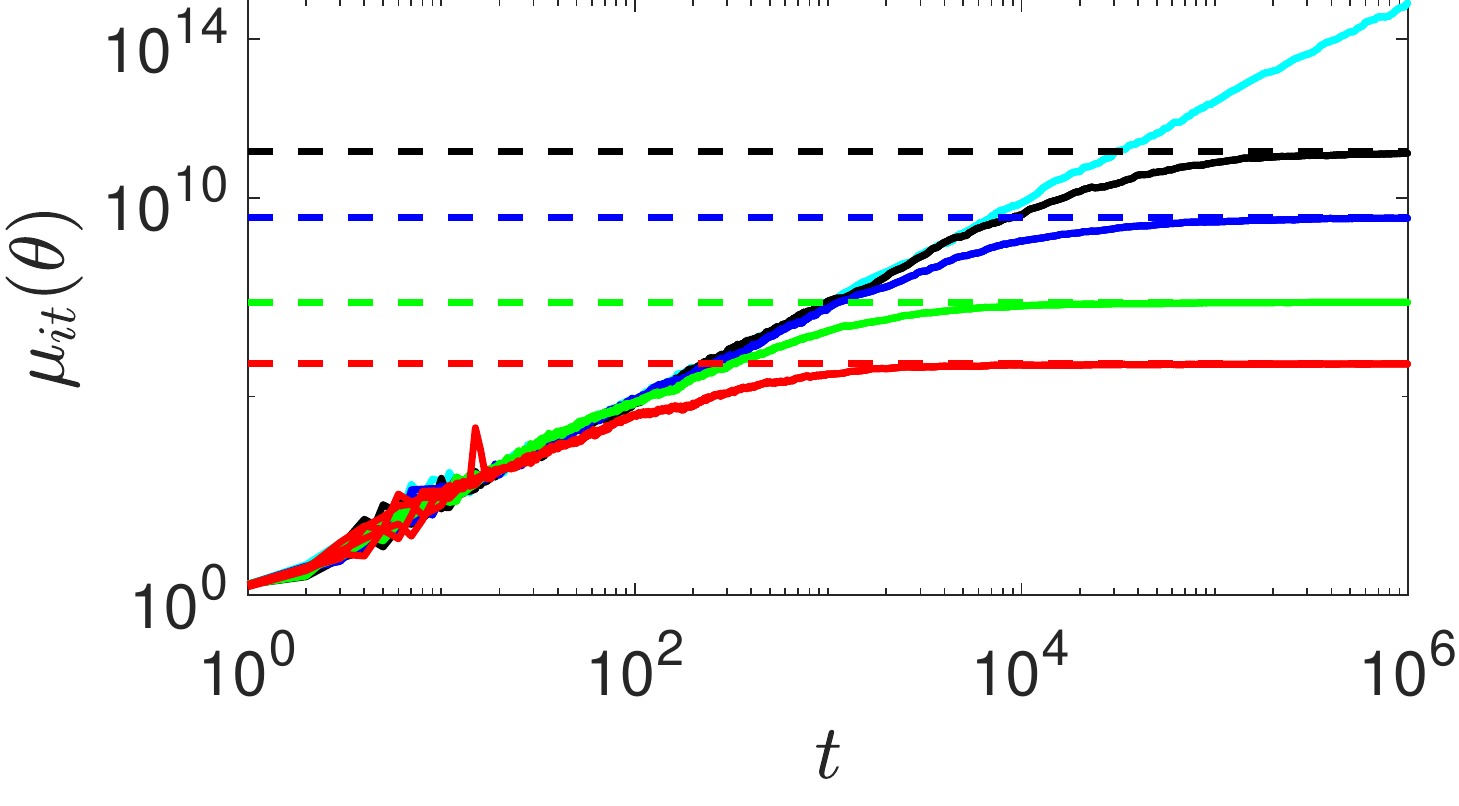}
    }%
    \subfigure[$\theta_2$]{
        \centering
        \includegraphics[width=0.32\textwidth]{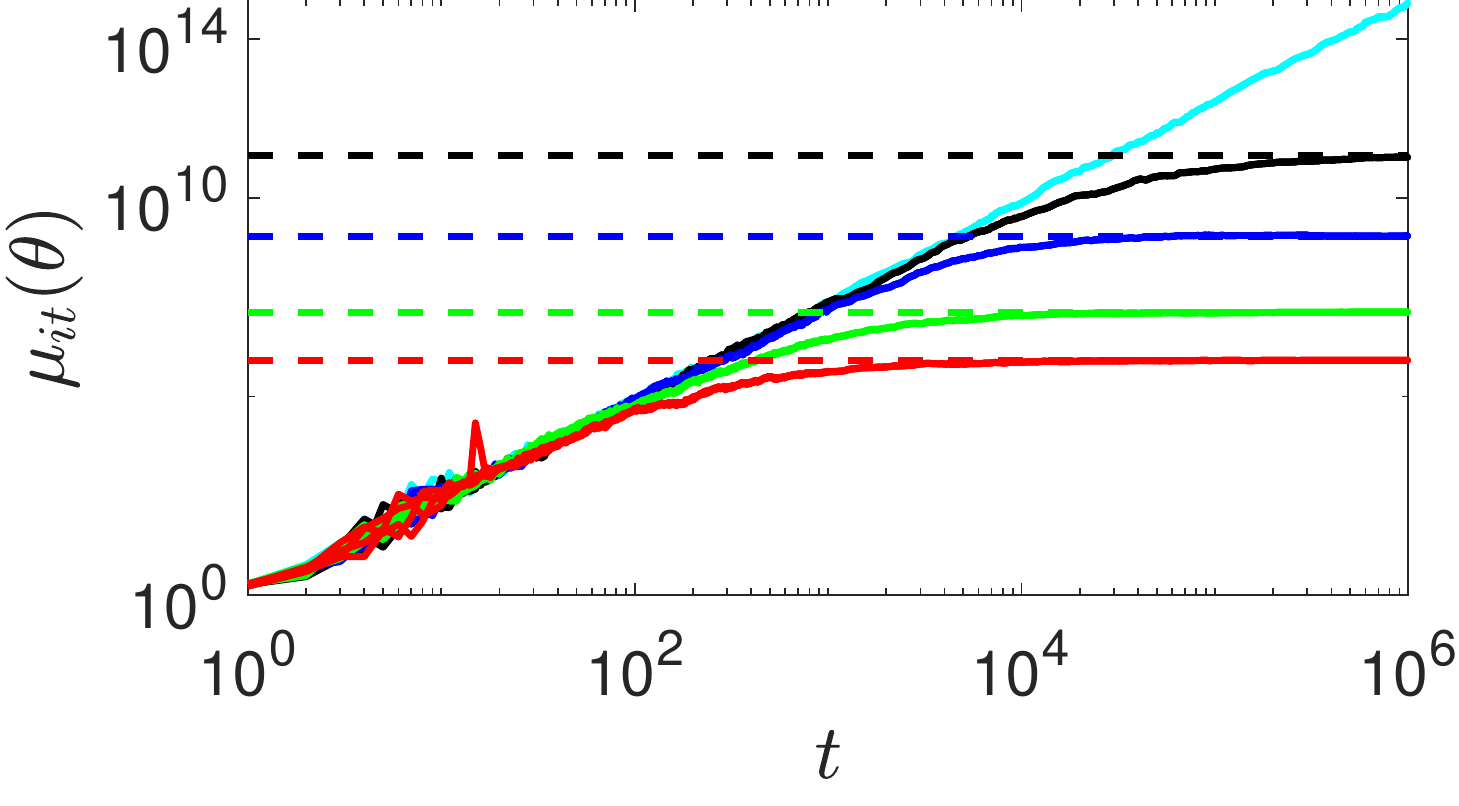}
    }%
    \subfigure[$\theta_3$]{
        \centering
        \includegraphics[width=0.32\textwidth]{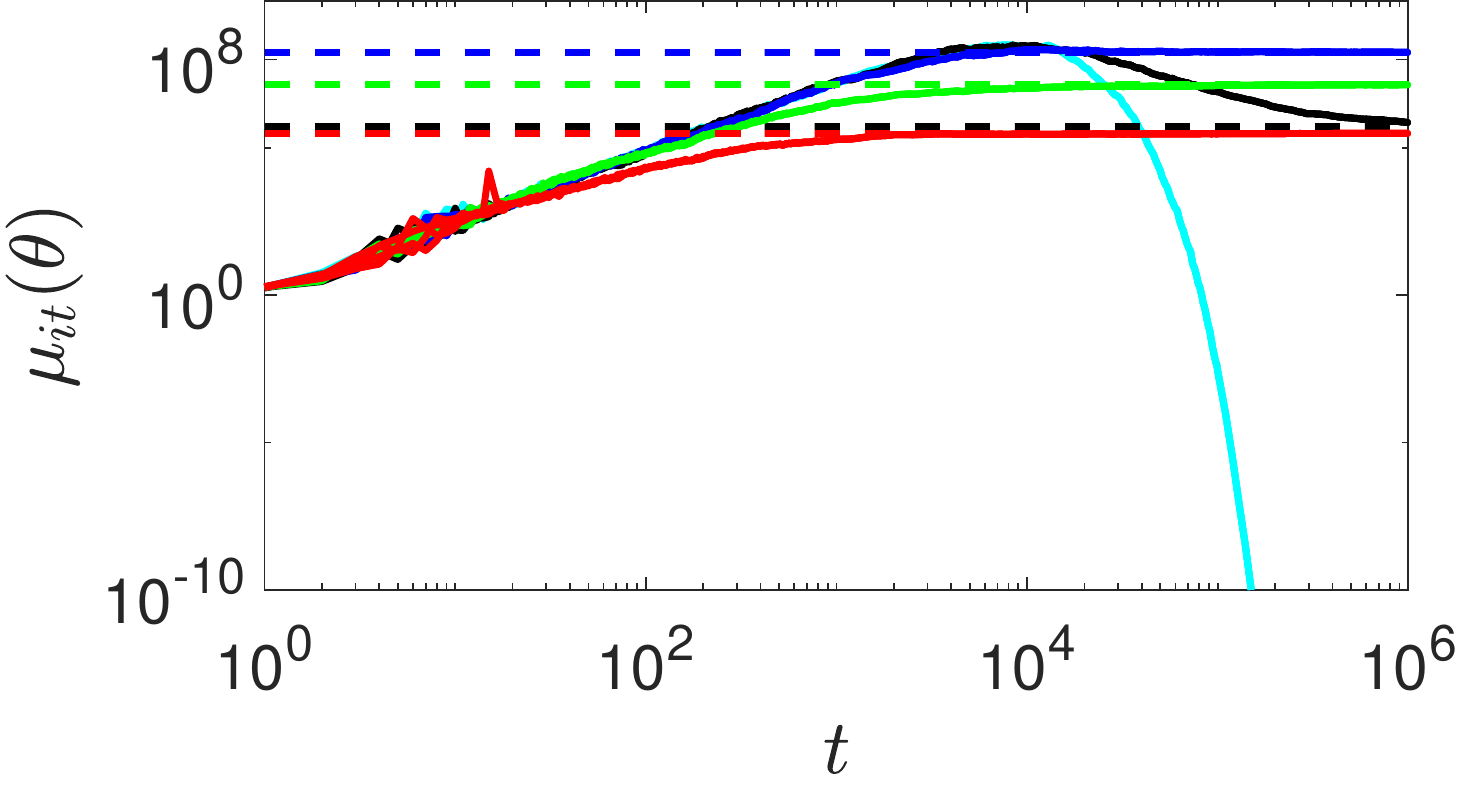}
    }

    \centering
    \subfigure{
        
        \includegraphics[width=0.75\textwidth]{leg_gauss_3-eps-converted-to.pdf}
    } \vspace{-6pt}
    \caption{Belief evolution of Gaussian uncertain models. The solid lines represent the beliefs $\mu_{t}^i(\theta)$, while the dashed lines represent the asymptotic point of convergence $(\prod_i^4 \widetilde{\Lambda}_{\theta}^i)^{(1/4)}$.} \vspace{-12pt}
    
    \label{fig:gauss_models}
\end{figure*}
\begin{table*}[]
\centering
\caption{Definition of mixture models.}\label{tab:mixtures}
\begin{tabular}{lc|lc|lc}
\multicolumn{2}{c|}{Mixture $Q_1$} & \multicolumn{2}{c|}{Mixture $Q_2$} & \multicolumn{2}{c}{Mixture $Q_3$}  \\
$\mathbf{m}_{11}=\begin{bmatrix} -1.5 & -1.5 \end{bmatrix}$' & \multirow{4}{*}{$\boldsymbol{\Sigma}_1 = \begin{bmatrix} 0.12 & 0 \\ 0 & 0.12 \end{bmatrix}$} & $\mathbf{m}_{21} = \begin{bmatrix} -1 & -1 \end{bmatrix}$' & \multirow{4}{*}{$\boldsymbol{\Sigma}_2 = \begin{bmatrix} 1.37 & 0 \\ 0 & 1.37 \end{bmatrix}$} & $\mathbf{m}_{31} = \mathbf{m}_{11}$' & \multirow{4}{*}{ $\boldsymbol{\Sigma}_3 = \begin{bmatrix} 0.25 & 0 \\ 0 & 0.25 \end{bmatrix}$} \\ 
$\mathbf{m}_{12} = \begin{bmatrix} -1.5 & 1.5 \end{bmatrix}$' &  & $\mathbf{m}_{22} = \begin{bmatrix} -1 & 1 \end{bmatrix}$' &  & $\mathbf{m}_{32} = \mathbf{m}_{12}$ &  \\
$\mathbf{m}_{13}=\begin{bmatrix} 1.5 & -1.5 \end{bmatrix}$' & & $\mathbf{m}_{23} = \begin{bmatrix} 1 & -1 \end{bmatrix}$' & & $\mathbf{m}_{33} = \mathbf{m}_{13}$ & \\
$\mathbf{m}_{14}=\begin{bmatrix} 1.5 & 1.5 \end{bmatrix}$' & & $\mathbf{m}_{24} = \begin{bmatrix} 1 & 1 \end{bmatrix}$' & & $\mathbf{m}_{34} = \mathbf{m}_{14}$ & \\
\multicolumn{2}{c|}{$p_1 = \begin{bmatrix} 0.25 & 0.25 & 0.25 & 0.25 \end{bmatrix}$} & \multicolumn{2}{c|}{$p_2 = p_1$} & \multicolumn{2}{c}{$p_3 = p_1$}  \\ \hline
\multicolumn{2}{c|}{Gaussian Fit of Mixture $Q_1$} & \multicolumn{2}{c|}{Gaussian Fit of Mixture $Q_2$} & \multicolumn{2}{c}{Gaussian Fit of Mixture $Q_3$}  \\
\multicolumn{1}{c}{$\tilde{\mathbf{m}}_1 = \begin{bmatrix} 0, 0 \end{bmatrix}$'} &  $\tilde{\boldsymbol{\Sigma}}_1 = \begin{bmatrix} 2.37 & 0 \\ 0 & 2.37 \end{bmatrix}$ & \multicolumn{1}{c}{$\tilde{\mathbf{m}}_2 = \begin{bmatrix} 0, 0 \end{bmatrix}$'} &  $\tilde{\boldsymbol{\Sigma}}_2 = \begin{bmatrix} 2.37 & 0 \\ 0 & 2.37 \end{bmatrix}$ & \multicolumn{1}{c}{$\tilde{\mathbf{m}}_3 = \begin{bmatrix} 0, 0 \end{bmatrix}$'} &  $\tilde{\boldsymbol{\Sigma}}_3 = \begin{bmatrix} 2.5 & 0 \\ 0 & 2.5 \end{bmatrix}$ 
\end{tabular} \vspace{-16pt}
\end{table*}

Figure~\ref{fig:gauss_models_ws} shows the ensemble average beliefs of each agent for each hypothesis. The first result seen is that the agents' beliefs are converging to the asymptotic point of convergence $(\prod_{i=1}^4 \tilde{\Lambda}_{i\theta})^{(1/4)}$ for each experiment and hypothesis, numerically indicating the correctness of Theorem~\ref{th:LL}. Additionally, as the amount of prior evidence grows, the beliefs for $\theta_1$ are diverging toward infinity, while the beliefs on $\theta_2$ and $\theta_3$ are decaying to $0$ at a rate of the average KL divergence, indicating the correctness of Theorem~\ref{cor:dogmatic}. 

Figure~\ref{fig:gauss_models_ws} also shows that when the agents have a low amount of evidence, the beliefs converge to value $>1$ and are considered consistent with the ground truth hypothesis and cannot be ruled out. When the agents acquire enough prior evidence for $\theta\ne\theta^*$, the beliefs eventually converge to a value $<1$ allowing them to identify hypotheses that are inconsistent with the ground truth. 

An attractive property of the beliefs generated by the uncertain likelihood ratio can be seen in Figure~\ref{fig:gauss_models_ws}(b). For prior evidence $|\mathbf{r}_\theta^i|\le 10^3$, the beliefs monotonically increase to the asymptotic point of convergence. Then, when $|\mathbf{r}_\theta^i|>10^3$, the agents have acquired enough evidence such that the beliefs ``peak-out" and reach a maximum value before decreasing. Even though they reach this peak value, it is still possible that the beliefs will converge to value $>1$, as indicated by the experiment with $|\mathbf{r}_\theta^i|\in [10^3,10^4]$. This property is also in Figure~\ref{fig:gauss_models_ws}(c), except here, as the KL divergence increases, the amount of prior evidence needed to evolve to the peak value decreases. This property indicates that the hypothesis with the largest peak value is the closest to $\theta^*$. This will be explored further with active learning approaches as future work. 

\begin{figure*}[t!]
    \subfigure[$\theta_1=\theta^*, 2\times 2$ Grid]{
        \centering
        \includegraphics[width=0.23\textwidth]{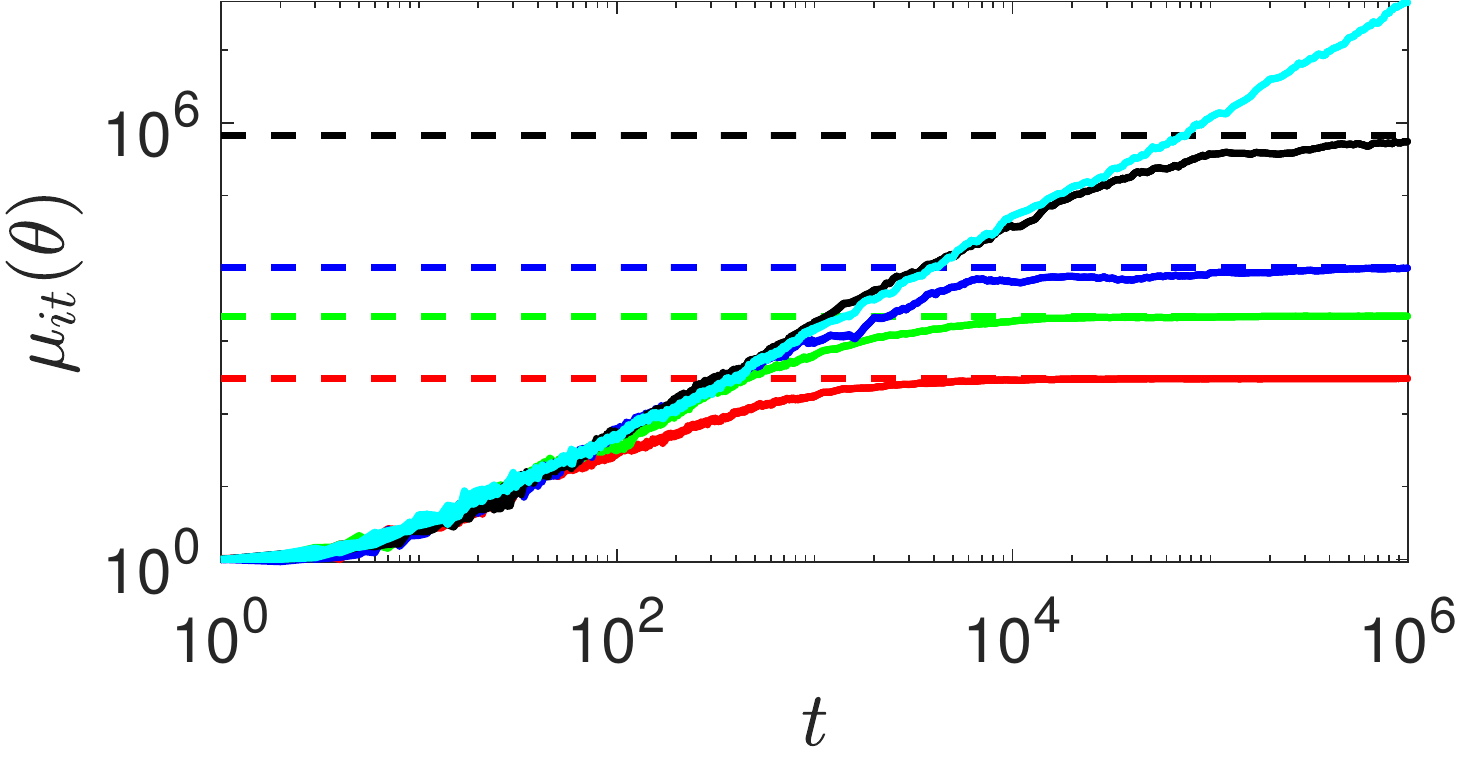}
    }%
    \subfigure[$\theta_1=\theta^*, 4\times 4$ Grid]{
        \centering
        \includegraphics[width=0.23\textwidth]{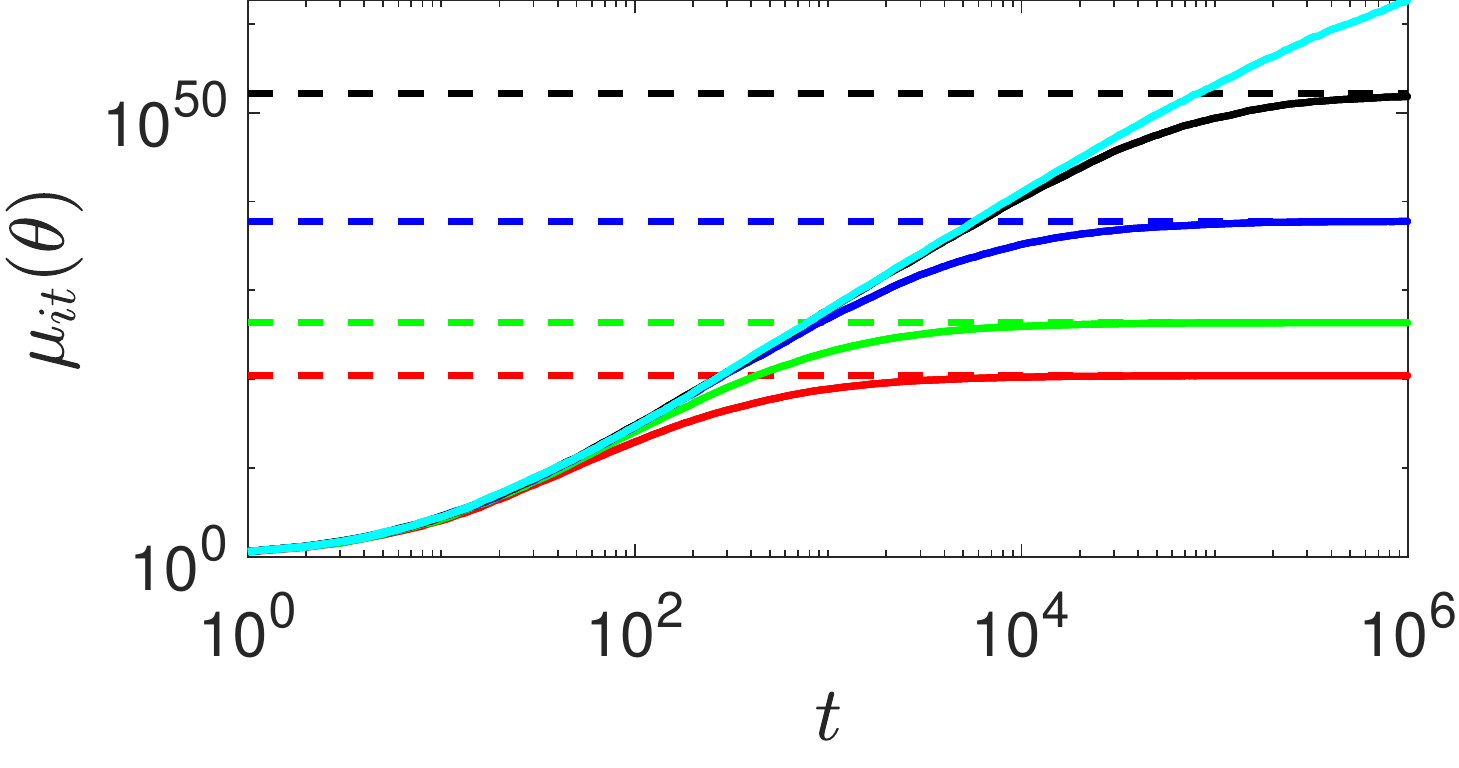}
    }%
    \subfigure[$\theta_1=\theta^*, 8\times 8$ Grid]{
        \centering
        \includegraphics[width=0.24\textwidth]{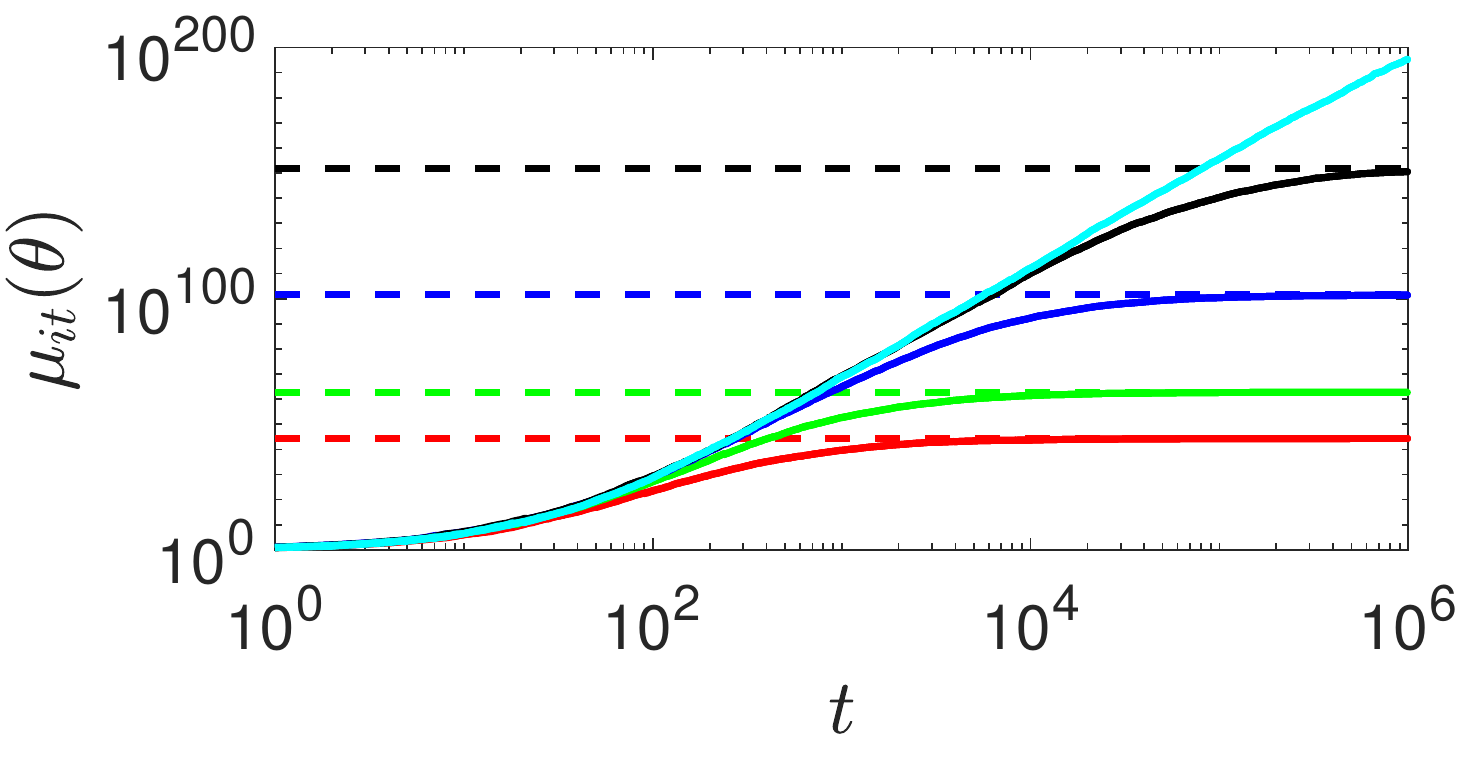}
    }%
    \subfigure[$\theta_1=\theta^*, 16\times 16$ Grid]{
        \centering
        \includegraphics[width=0.24\textwidth]{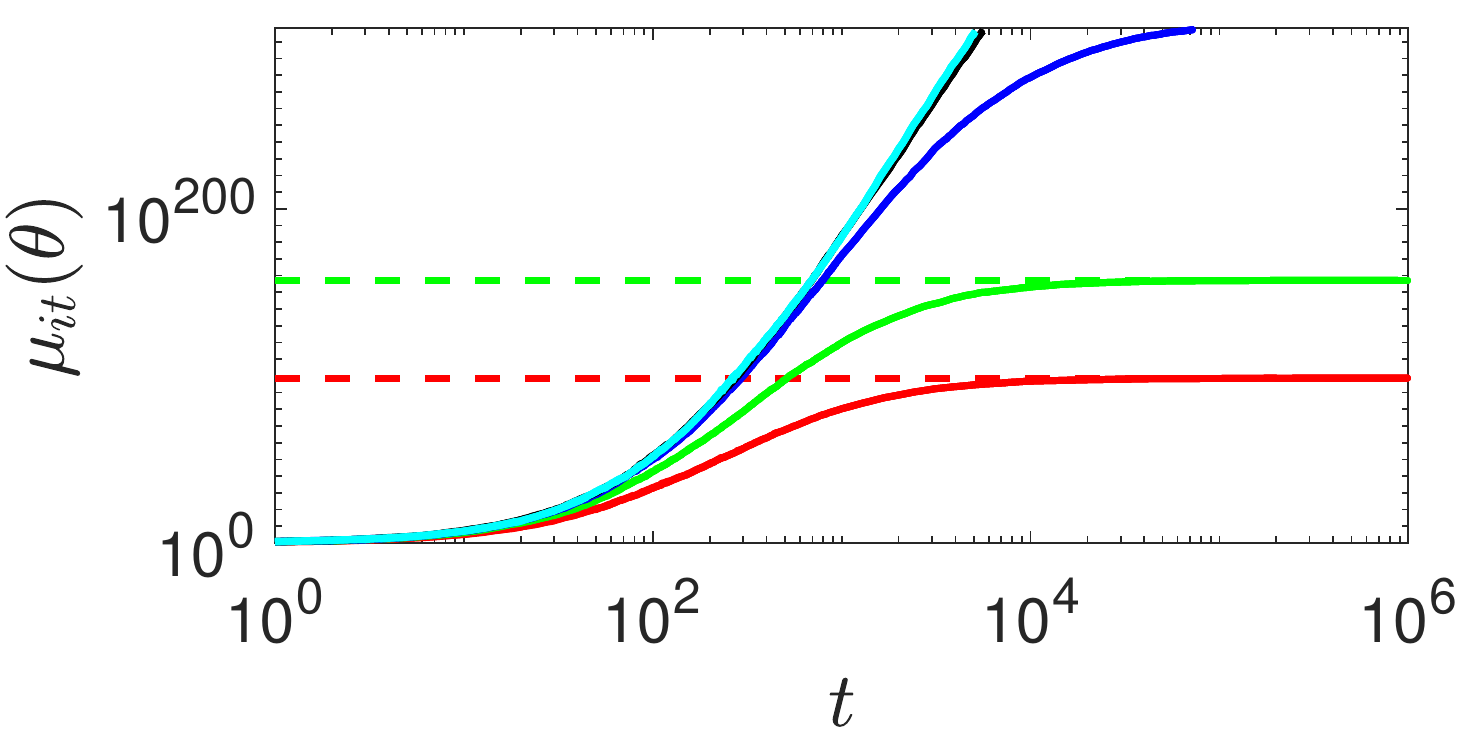}
    }
    \subfigure[$\theta_2, 2\times 2$ Grid]{
        \centering
        \includegraphics[width=0.23\textwidth]{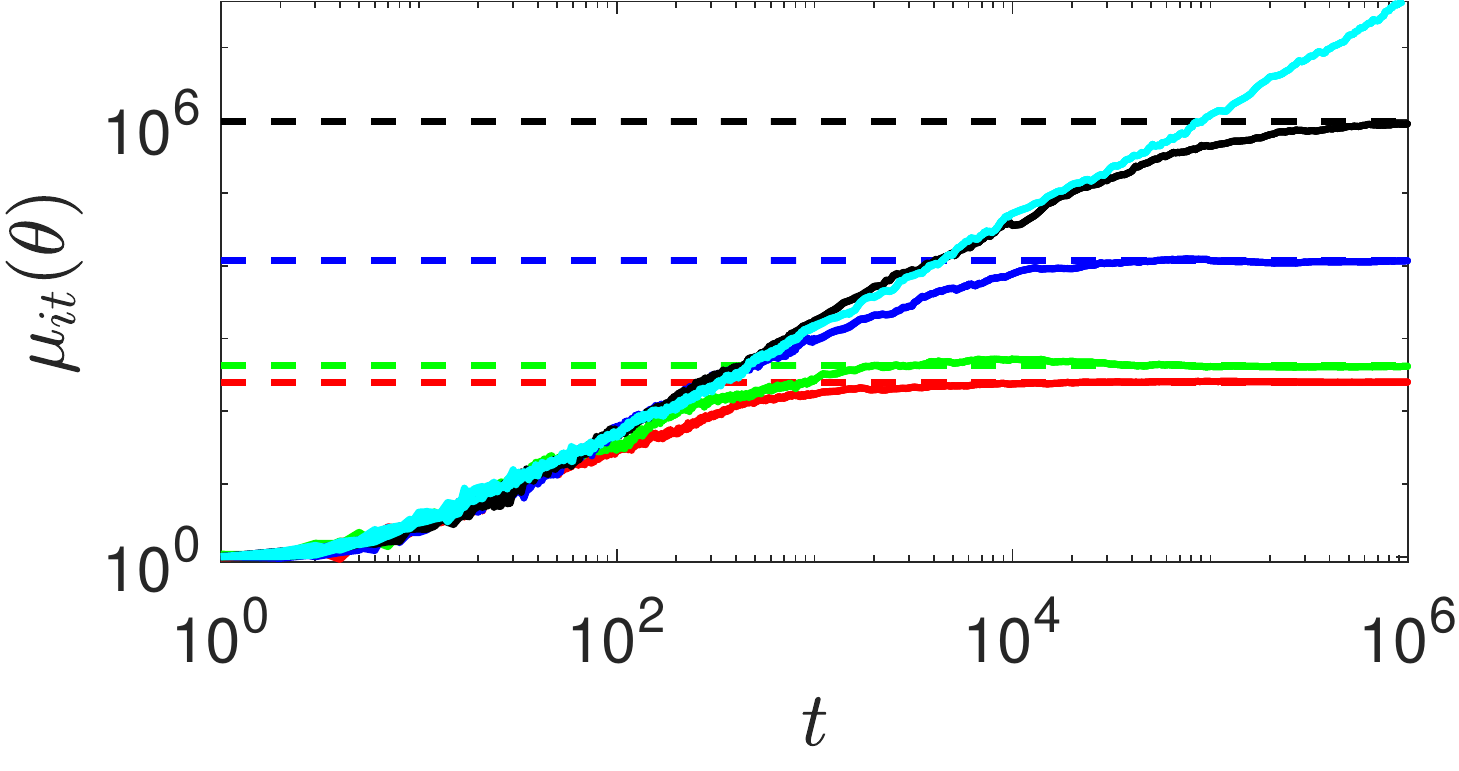}
    }%
    \subfigure[$\theta_2, 4\times 4$ Grid]{
        \centering
        \includegraphics[width=0.24\textwidth]{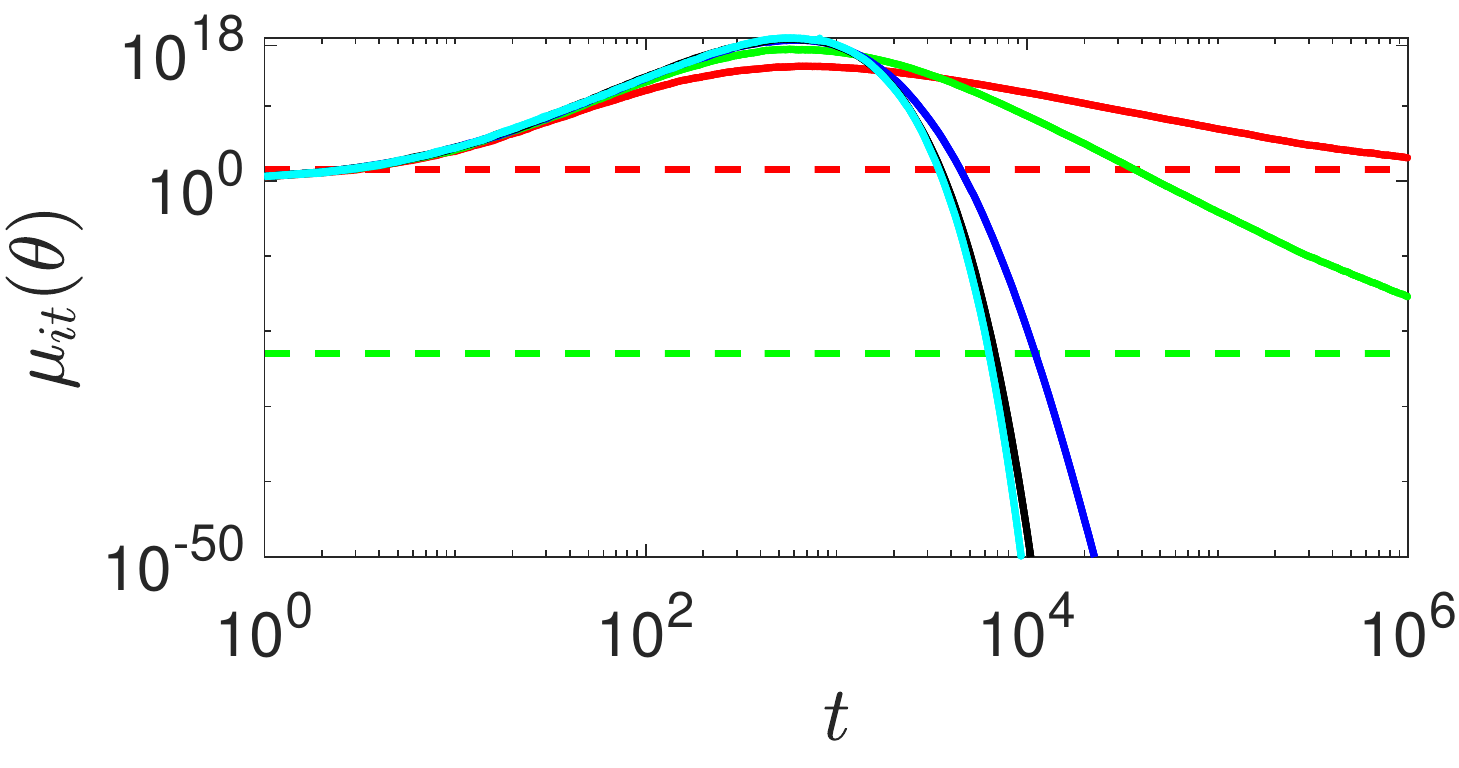}
    }%
    \subfigure[$\theta_2, 8\times 8$ Grid]{
        \centering
        \includegraphics[width=0.24\textwidth]{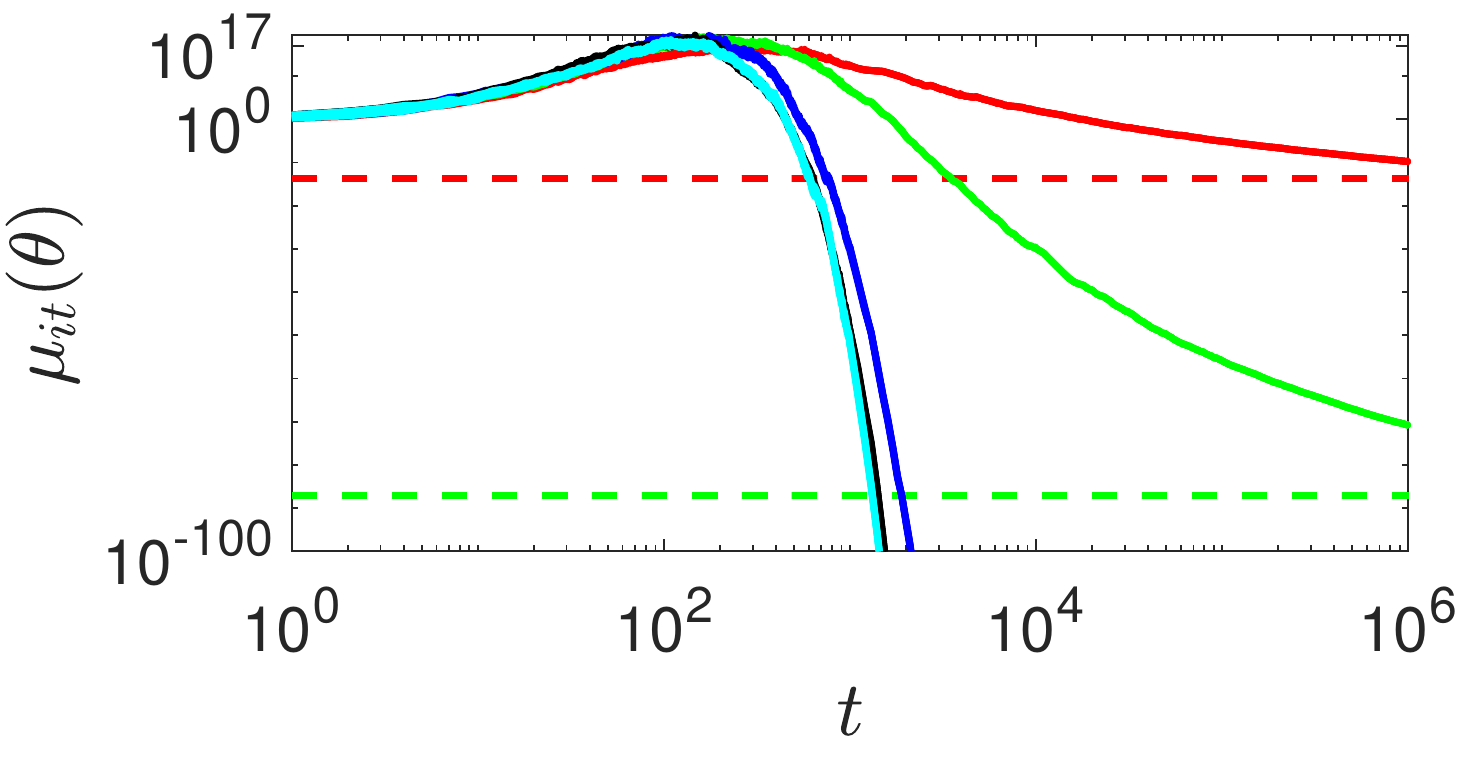}
    }%
    \subfigure[$\theta_2, 16\times 16$ Grid]{
        \centering
        \includegraphics[width=0.24\textwidth]{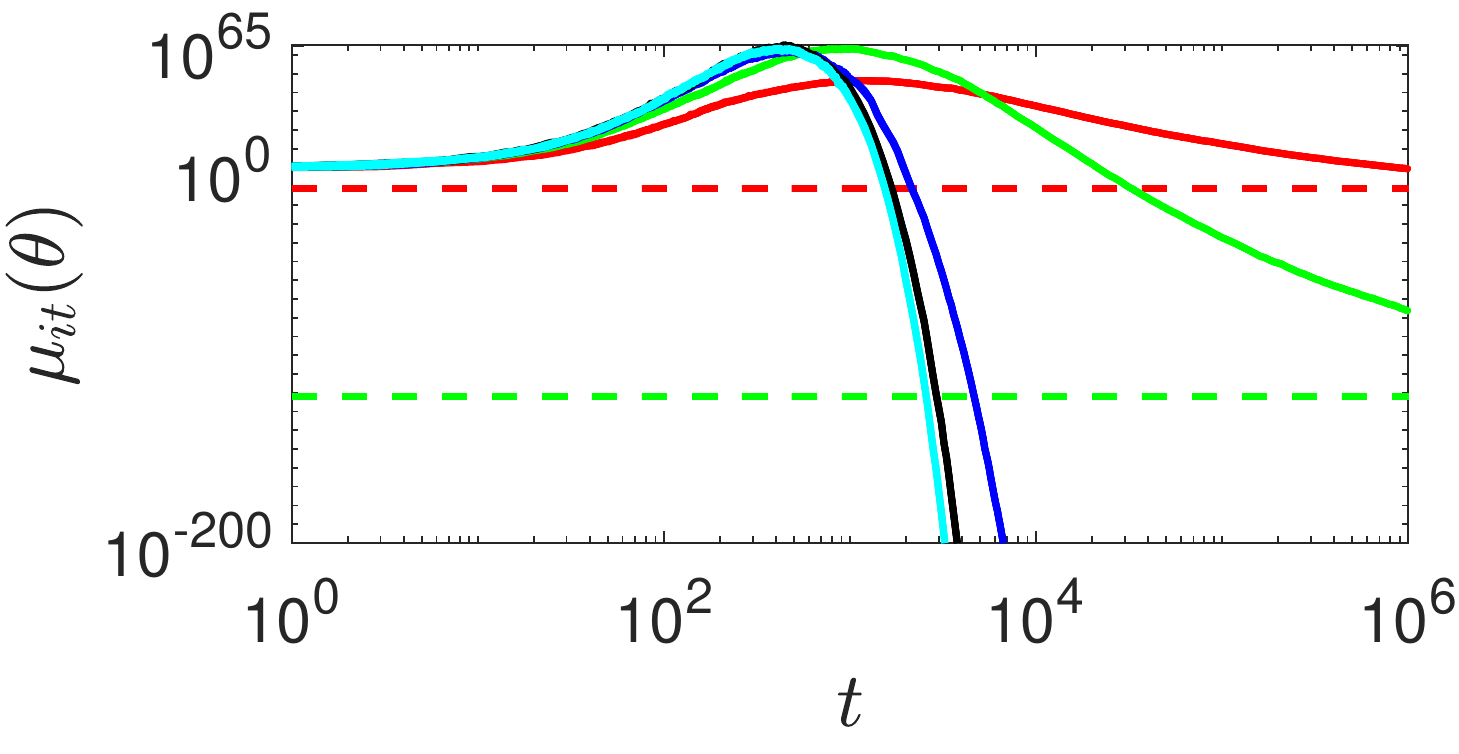}
    }
    \subfigure[$\theta_3, 2\times 2$ Grid]{
        \centering
        \includegraphics[width=0.23\textwidth]{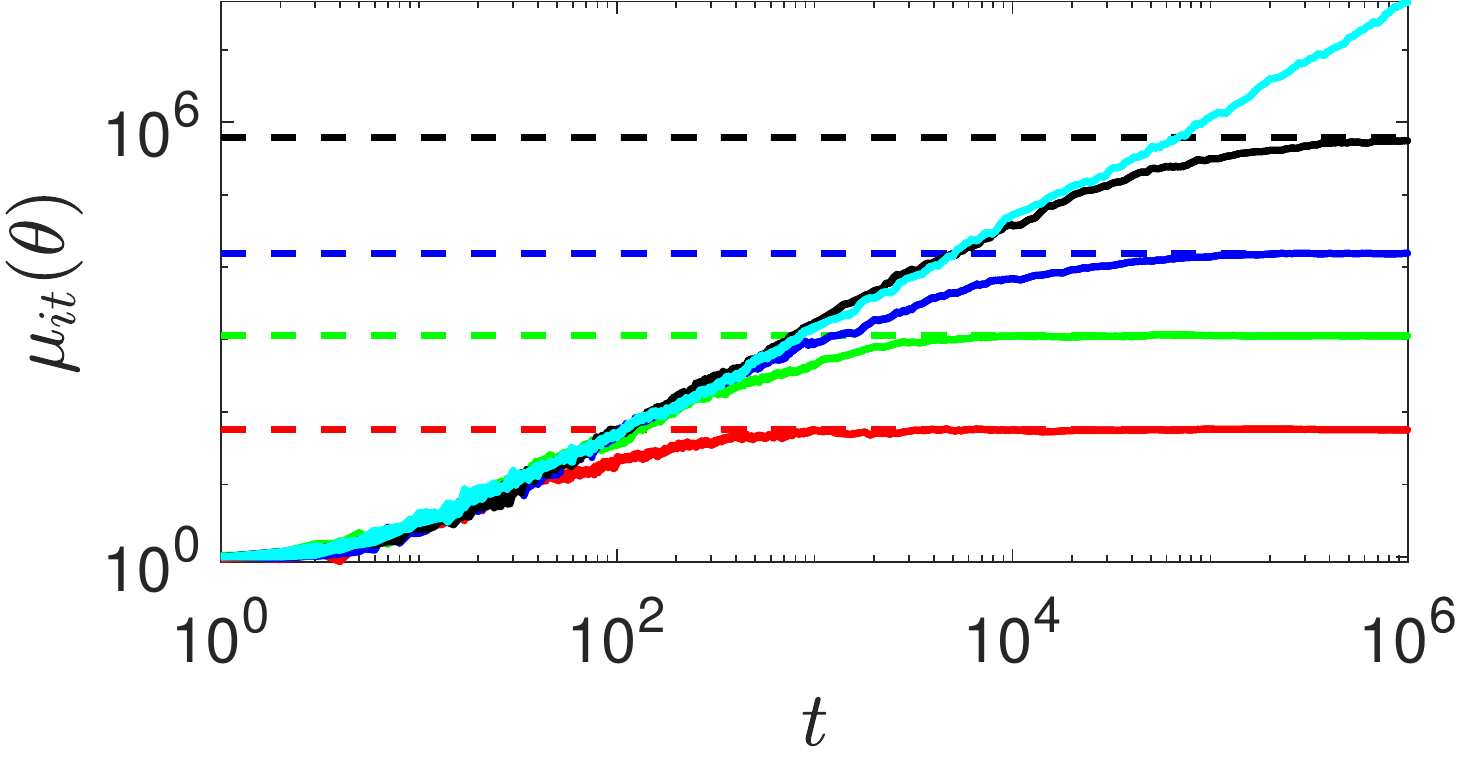}
    }%
    \subfigure[$\theta_3, 4\times 4$ Grid]{
        \centering
        \includegraphics[width=0.24\textwidth]{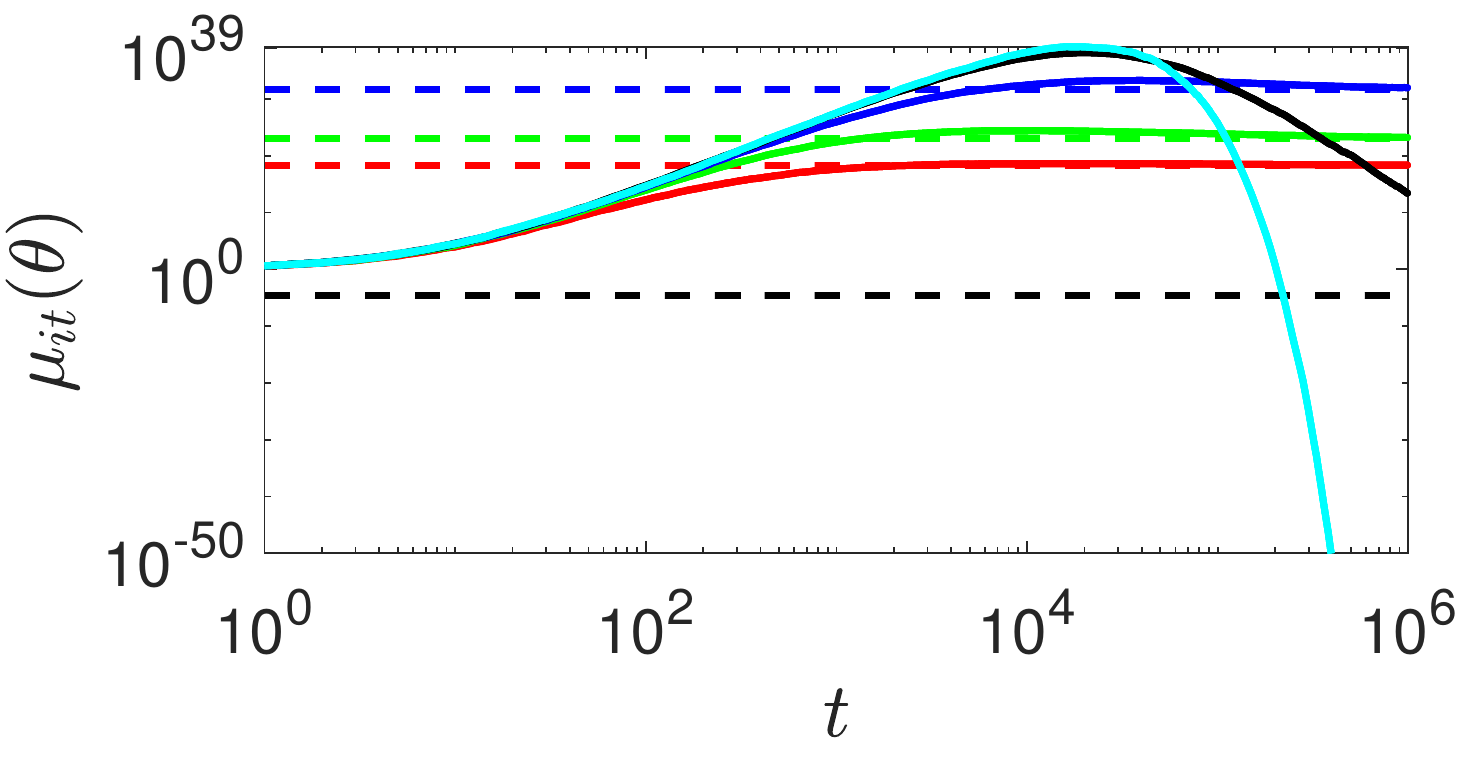}
    }%
    \subfigure[$\theta_3, 8\times 8$ Grid]{
        \centering
        \includegraphics[width=0.24\textwidth]{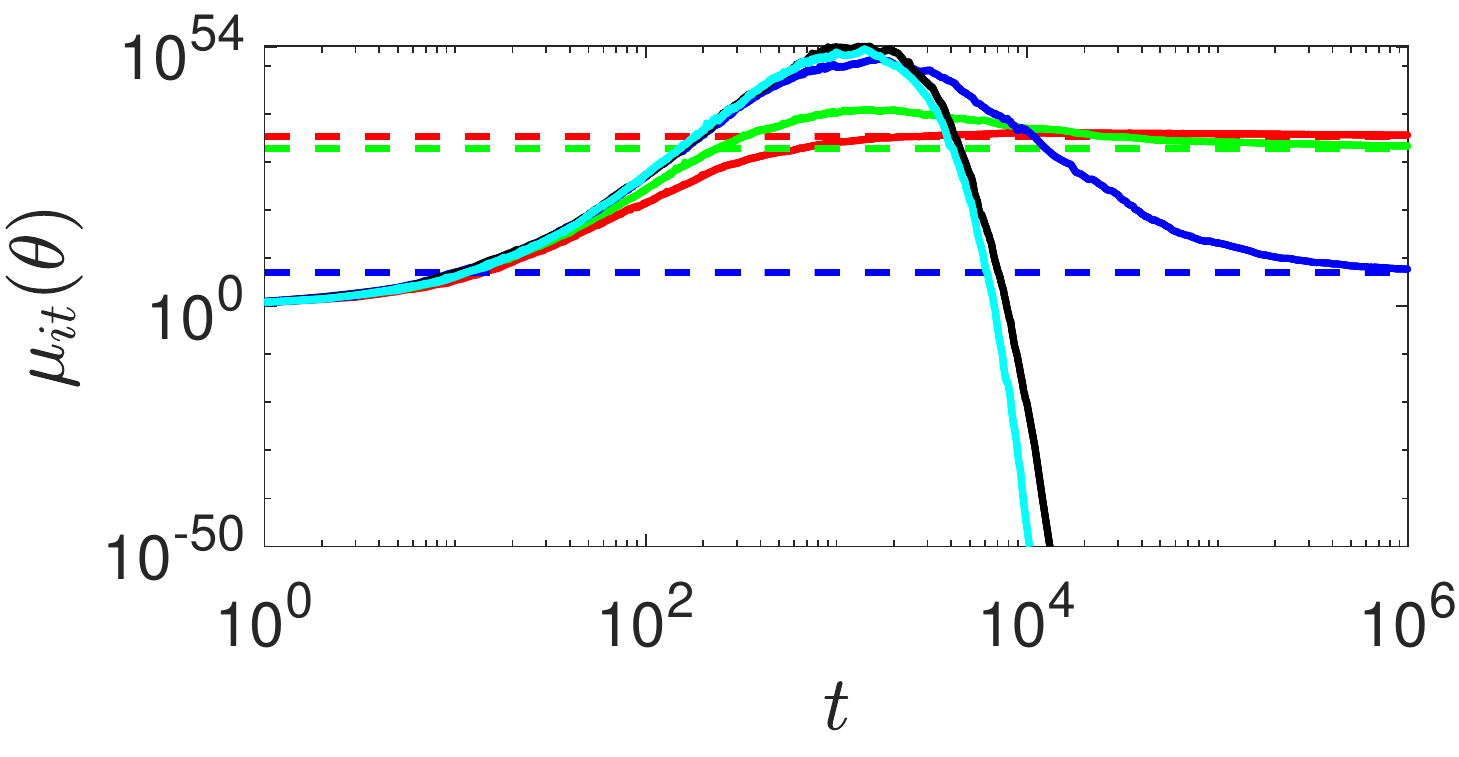}
    }%
    \subfigure[$\theta_3, 16\times 16$ Grid]{
        \centering
        \includegraphics[width=0.24\textwidth]{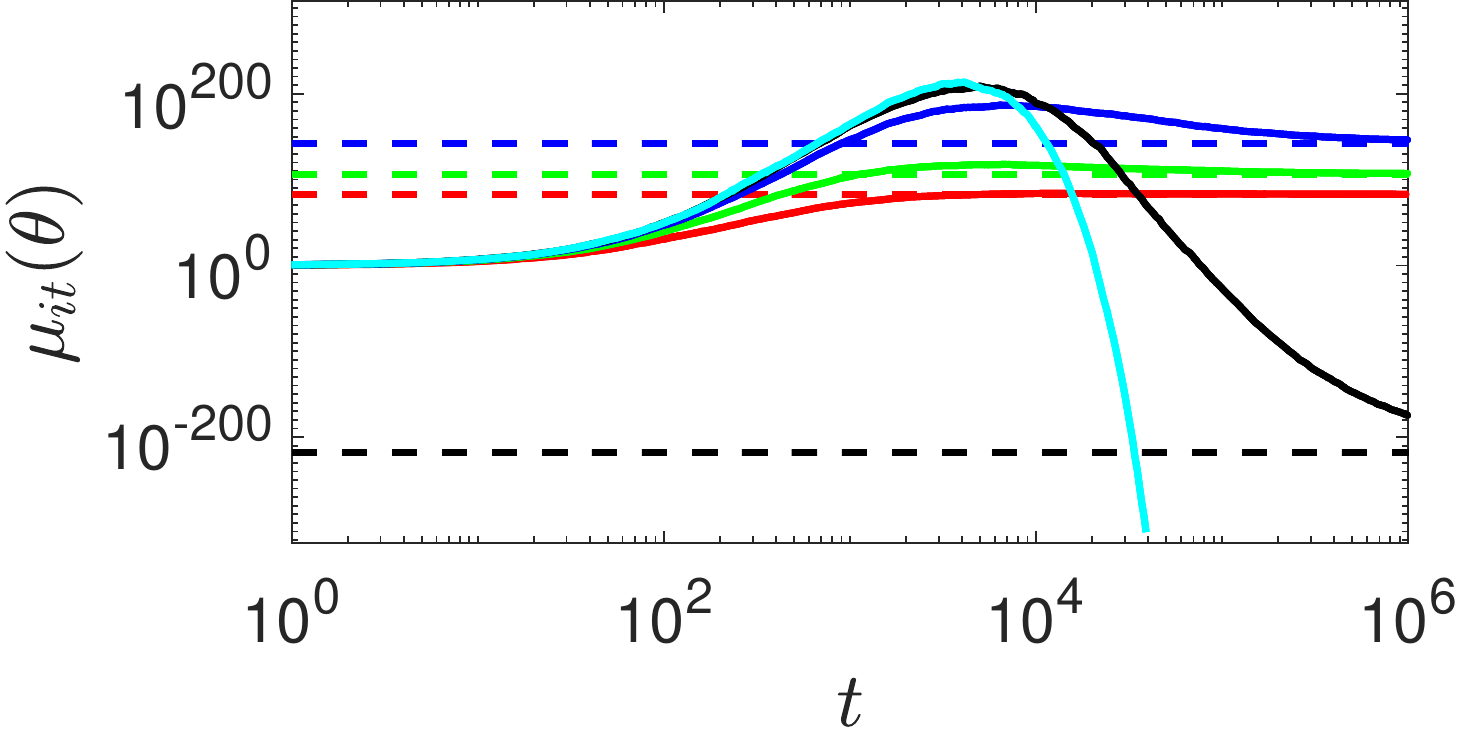}
    }
    \centering
    \subfigure{
        
        \includegraphics[width=0.65\textwidth]{leg_gauss_3-eps-converted-to.pdf}
    } \vspace{-6pt}
    \caption{Belief evolution of multinomial uncertain models utilizing the nonparametric approach presented in Section~\ref{sec:example_non}. The solid lines represent the beliefs $\mu_{t}^i(\theta)$ for various grid sizes, while the dashed lines represent the asymptotic point of convergence $(\prod_i^4 \widetilde{\Lambda}_{\theta}^i)^{(1/4)}$.} \vspace{-12pt}
    
    \label{fig:np_models}
\end{figure*}
\subsection{Misspecified likelihood models}
In this experiment, we followed the same setup as before, however, we assume that the three possible states of the world, $Q_1$, $Q_2$, and $Q_3$ are distribution according to a four-component Gaussian mixture model, i.e., $Q_m = \sum_{c=1}^4 p_{m,c}\mathcal{N}(\mathbf{m}_{m,c},\boldsymbol{\Sigma}_{m})$, where $p_{m,c}$ is the mixture probabilities. The parameters of the mixture models considered are defined in Table~\ref{tab:mixtures}. First, we consider that the agents make a naive assumption that the observations are drawn from the family of Gaussian parametric models, although the possible hypotheses are multi-modal. The Gaussian fit of each mixture is shown in Table~\ref{tab:mixtures} and represents the parameters that minimize the KL divergence between the underlying mixture and a Gaussian distribution. In this setting, the agents implement the uncertain likelihood update~\eqref{eq:mvn_ell} presented in Section~\ref{sec:mvn_models}. 

The belief evolution for the misspecified Gaussian models is shown in Figure~\ref{fig:gauss_models}. As seen, all of the properties discussed for well-specified Gaussian models hold here, numerically indicating that Theorems~\ref{th:LL} and~\ref{cor:dogmatic} hold. However, due to the abstraction of the underlying distribution, the agents cannot distinguish between $\theta_1$ and $\theta_2$ for any amount of prior evidence because both hypotheses exhibit the same mean and covariance. Therefore, when the assumed parametric family of distributions is misspecified, modeling uncertainty may result in confusion between hypotheses with a probability of $>0$. 

To overcome this issue, we then implemented the non-parametric approach presented in Section~\ref{sec:example_non}, where we partition the observation space and treat the cell occupants as samples drawn from an unknown multinomial distribution. We varied the number of rectangular cuboids between $K=\{4, 16, 64, 256\}$ by constructing a $2\times 2$, $4\times 4$, $8\times 8$, and $16\times 16$ grid as shown in Figure~\ref{fig:grid_partitions}. Here, we set the center of the grid to $(0,0)$ and the rectangle $R^i=[-3,-3]\times [3,3]$ for all $i\in\mathcal{M}$. The boundary hyperplanes were chosen to encompass most of the samples generated from the distribution for each hypothesis. Then, the agents utilize the uncertain likelihood update~\eqref{eq:multinomial_ell}.

Figure~\ref{fig:np_models} presents the evolution of the beliefs for each hypothesis and each grid. The rows represent a single hypothesis, while the columns represent the grid size. Starting from the left, we can see that the beliefs with the coarsest grid size for all hypotheses, i.e., $2\times 2$, result in the agents not distinguishing the ground truth. Then, as one increases the number of grids, the maximum value of the beliefs increases for each experiment. This is because of the slope increases as the dimension of the multinomial distribution increases. 

In this case, as the grids become finer, the non-parametric approach can distinguish the ground truth from $\theta_2$, unlike the naive Gaussian assumption presented in Figure~\ref{fig:gauss_models}. However, the beliefs of $\theta_3$ require sufficient prior evidence and finer grid size to distinguish it from the ground truth. With a $4\times 4$ grid and finite prior evidence, the beliefs converge to a value of $>1$ or close to $1$, indicating that the hypothesis may be consistent with the ground truth. Once we move to the $8\times 8$ grid, the agents can distinguish $\theta_3$ for the case when $|\mathbf{r}_\theta^i|>10^3$, indicating that a finer grid size might result in the beliefs converging to a smaller number. However, when we implement the $16\times 16$ grid for $\theta_3$, more prior evidence is required to distinguish the hypothesis from the ground truth. The grid size must be fine enough to distinguish the prior evidence but not too fine to require more than the available prior evidence. Furthermore, the computational complexity increases for finer grids. The selection of best grid size in terms of type I and II errors will be investigated as future work.

\section{Conclusion and Future Work} \label{sec:con}
This paper generalizes earlier work on uncertain models~\cite{HUKJ2020_TSP}, which utilizes the theory of non-Bayesian social learning to provide a distributed hypothesis testing framework with a network of heterogeneous agents receiving a stream of observations from a general class of measurement distributions. We assume that the agents have limited amounts of training data to learn the parameters of the likelihood models for each hypothesis, requiring uncertain models to act as their surrogate likelihood function. We first present the general uncertain models and identify a set of regularity conditions that guarantee that the agents will learn the true likelihood model asymptotically, while accurately encompassing their epistemic uncertainty when training data is limited. Then, we implemented the uncertain models into non-Bayesian social learning and showed that the agents could infer the true state of the world asymptotically. Additionally, we show that as the agents become certain about the parameters of the likelihood models, the beliefs are consistent with traditional non-Bayesian social learning, and the network learns the ground truth hypothesis. We then provided a detailed road map to guide the network designer to implement uncertain models and examples for continuous and discrete measurement distributions. Finally, we provided an initial study on extending the uncertain models to non-parametric measurement distributions and showed numerically that our framework provides the same properties as the parametric case. 

As future work, we will expand upon the non-parametric framework to further study the main challenges presented at the end of Section~\ref{sec:example_non}. This includes identifying the trade-offs in grid design and the number of bins based on Type I and Type II errors and providing error bounds. Additionally, we will investigate active learning approaches that utilize these error bounds to determine if an agent must collect more prior evidence to distinguish the ground truth hypothesis.

\bibliographystyle{IEEEtran}
\bibliography{ref}

\appendices
\section{Proof of Theorem~\ref{th:LL}}
First, we present the following lemma to show that the general models have the same properties as the multinomial and Gaussian uncertain models presented in \cite{HUKJ2020_TSP, HUKJ2020_Gauss}

\begin{lemma} \label{lem:ulru_finite}
Let Assumption~\ref{assum:support}-\ref{assum:reg} hold. Then, the uncertain likelihood update for each hypothesis $\theta\in\boldsymbol{\Theta}$ and each agent $i\in\mathcal{M}$ have the following properties: 
\begin{align}
    \lim_{t\to\infty} \ell_\theta^i(\omega|\boldsymbol{\omega}_{1:t-1}^i) = \left\{ \begin{array}{cl}
        1 & \text{if} \ |\mathbf{r}_\theta^i|<\infty, \\
        \frac{P^i(\omega|\boldsymbol{\phi}_\theta^i)}{P^i(\omega|\boldsymbol{\phi}_{\theta^*}^i)} & \ \text{otherwise},
    \end{array} \right.
\end{align}
in probability. 
\end{lemma}

\begin{proof}
The first condition when $|\mathbf{r}_\theta^i|<\infty$ is easily obtained since the uncertain likelihood update converges in probability to 
\begin{align}
    \lim_{t\to\infty}\ell_{\theta}^i(\omega_{t}^i)&=\frac{\int_{\boldsymbol{\Phi}^i} P^i(\omega_{t}^i|\boldsymbol{\phi}^i)\delta_{\boldsymbol{\phi}_{\theta^*}^i
}(\boldsymbol{\phi}^i) d\boldsymbol{\phi}^i}{\int_{\boldsymbol{\Phi}^i} P^i(\omega_{t}^i|\boldsymbol{\phi}^i)\delta_{\boldsymbol{\phi}_{\theta^*}^i}(\boldsymbol{\phi}^i) d\boldsymbol{\phi}^i} = \frac{P^i(\omega_{t}^i|\boldsymbol{\phi}_{\theta^*}^i)}{P^i(\omega_{t}^i|\boldsymbol{\phi}_{\theta^*}^i)}, \nonumber
\end{align}
according to Proposition~\ref{lem:normal_posterior} property (a). Similarly, the second condition is trivially obtained due to property (b) of Proposition~\ref{lem:normal_posterior}.  
\end{proof}

Next, we include the following lemmas to help prove Theorem~\ref{th:LL}.
Finally, we include the following two lemmas. 

\begin{lemma}[Lemma 5 in~\cite{NOU2017}] \label{lem:doubly}
For a stationary doubly stochastic matrix, we have for all $t>0$
\begin{eqnarray}
\left \| \mathbf{A}^t - \frac{1}{m} \mathbf{11}' \right\| \le \sqrt{2}m \lambda^t
\end{eqnarray}
	where $\|\cdot\|$ is the spectral norm, $\lambda= 1-\frac{\eta}{4m^2}$, and $\eta$ is a positive constant s.t. if $[\mathbf{A}]_{ij}>0$, then $[\mathbf{A}]_{ij}\ge \eta$.
\end{lemma}

\begin{lemma}[Lemma $3.1$ in \cite{ram10}]\label{lemm:ram}
		Let $\{\gamma_k \}$ be a scalar sequence. If  $\lim_{k \to \infty} \gamma_k = \gamma$ and $0\leq \beta \leq 1$, then $\lim_{k\to \infty} \sum_{l=0}^{k} \beta^{k{-}l} \gamma_l = \frac{\gamma}{1{-}\beta}$.
\end{lemma}

\begin{proof}[Proof of Theorem~\ref{th:LL}]
We start by taking the standard $2$-norm between the log-beliefs and the log-average asymptotic uncertain likelihood ratios, i.e.,
\begin{align}
    &\Big\| \log(\boldsymbol{\mu}_{t+1}(\theta)) - \frac{1}{m}\sum_{i=1}^m\log(\Lambda_{i\theta}(t+1))\mathbf{1}\Big\| \nonumber \\
    &= \bigg\| \left(1-\frac{1}{m}\right)\mathbf{I}\log(\boldsymbol{\ell}_{\theta}(t+1)) + \nonumber \\
    & \ \ \ \ \ \  \sum_{\tau=1}^t\left(\mathbf{A}^{t-\tau}-\frac{1}{m}\mathbf{11'}\right)\log(\boldsymbol{\ell}_{\theta}(\tau))\bigg\| \nonumber \\
    &\le \big\|\log(\boldsymbol{\ell}_{\theta}(t+1)) \big\| + \sum_{\tau=1}^t\bigg\|\mathbf{A}^{t-\tau}-\frac{1}{m}\bigg\| \big\|\log(\boldsymbol{\ell}_{\theta}(\tau))\big\|  \nonumber \\
    &\le  \sqrt{2}m\left(\sum_{\tau=0}^t \lambda^{t-\tau}\big\|\log(\boldsymbol{\ell}_{\theta}(\tau))\big\| - \lambda^t\big\|\log(\boldsymbol{\ell}_{\theta}(0))\big\|\right),
\end{align}
where $\mathbf{I}$ is the identity matrix and $\boldsymbol{\ell}_{\theta}(\tau) = [\ell_{\theta}^1(\omega_{\tau}^1),...,\ell_{\theta}^m(\omega_{\tau}^m)]'$. The first inequality follows since the norm of $(1-(1/m))\mathbf{I}$ is less than $1$, while the the final inequality is achieved using Lemma~\ref{lem:doubly}. Furthermore, since $\lim_{t\to\infty} \left\|\log\left(\boldsymbol{\ell}_{\theta}(t)\right)\right\|=0$ in probability from Lemma~\ref{lem:ulru_finite}, then 
\begin{eqnarray}
\lim_{t\to\infty} \sum_{\tau=0}^{t} \lambda^{t-\tau}  \left\|\log\left(\boldsymbol{\ell}_{\theta}(\tau)\right)\right\| = 0 \nonumber
\end{eqnarray}
in probability from Lemma~\ref{lemm:ram}. Finally, since $\lambda<1$ from Lemma~\ref{lem:doubly} and $\left\|\log\left(\boldsymbol{\ell}_{\theta}(0)\right)\right\|$ is bounded according to Lemma~\ref{lem:ell_bounded}
\begin{eqnarray}
    \lim_{t\to\infty} \lambda^t \left\|\log\left(\boldsymbol{\ell}_{\theta}(0)\right)\right\| = 0 \ \ \ \text{in probability.} \nonumber
\end{eqnarray}
Then, by the continuity of the logarithmic function, this implies that $\lim_{t\to\infty} \boldsymbol{\mu}_t(\theta)\oslash \left(\left( \prod_{j=1}^m \Lambda_{\theta}^j(t)\right)^{1/m}\right)\mathbf{1} = \mathbf{1}$ in probability and the desired result is achieved, where $\oslash$ indicates the Hadamard division.

\end{proof}

\section{Proof of Theorem~\ref{cor:dogmatic}}
Next, we prove Theorem~\ref{cor:dogmatic} where $|\mathbf{r}_\theta^i|\to\infty$ $\forall i\in\mathcal{M}$. First, we present the following corollary and lemmas that show that the general uncertain models are consistent with multinomial and Gaussian uncertain models presented in~\cite{HUKJ2020_TSP, HUKJ2020_Gauss}. 

\begin{corollary} \label{cor:ell_matching}
Let Assumptions~\ref{assum:support}-\ref{assum:reg} hold and agent $i$ is certain for every hypothesis $\theta\in\boldsymbol{\Theta}_i^*$, i.e., $|\mathbf{r}_\theta^i|\to\infty$. Then, the uncertain likelihood update $\ell_{\theta}^i(\omega_{t}^i)$ converges in probability to $1$ as $t\to\infty$.%
\end{corollary}

\begin{lemma} \label{cor:exp_log_ell}
Let Assumptions~\ref{assum:support}-\ref{assum:reg} hold and the agent $i$ is certain about a hypothesis $\theta\in\boldsymbol{\Theta}$, i.e., $|\mathbf{r}_\theta^i|\to\infty$. Then, the expected value of the log-uncertain likelihood update has the following properties,
 \begin{align}\label{eq:exp_bound_ell}
\mathbb{E}_{Q^i}[\log(\ell_{\theta}^i(\omega_{t}^i))] = D_{KL}&\Big(Q^i\Big\|\hat{P}^i(\omega_{t}^i|\boldsymbol{\omega}_{1:t-1}^i))\Big)- \nonumber \\ &D_{KL}\Big(Q^i\Big\|P^i(\omega_{t}^i|\boldsymbol{\phi}_{i\theta}))\Big), 
\end{align}
in probability, where $\hat{P}^i(\omega_{t}^i|\boldsymbol{\omega}_{1:t-1}^i)=\int_{\boldsymbol{\Phi}^i} P^i(\omega_{t}^i|\boldsymbol{\phi}^i)f(\boldsymbol{\phi}^i|\psi(\boldsymbol{\omega}_{1:t-1}^i))d\boldsymbol{\phi}$ is the denominator of the uncertain likelihood update~\eqref{eq:ulu}, and
 \begin{multline}
\lim_{t\to\infty} \mathbb{E}_{Q^i}[\log(\ell_{\theta}^i(\omega_{t}^i))] = D_{KL}\Big(Q^i\Big\|P^i(\cdot|\boldsymbol{\phi}_{\theta^*}^i))\Big) \\- D_{KL}\Big(Q^i\Big\|P^i(\cdot|\boldsymbol{\phi}_\theta^i))\Big),
\end{multline}
in probability.

\end{lemma}

\begin{proof}
First, using condition (b) of Proposition~\ref{lem:normal_posterior}, as $|\mathbf{r}_\theta^i|\to\infty$, the uncertain likelihood update converges in probability to 
\begin{align} \label{eq:ell_r_infty}
\lim_{|\mathbf{r}_\theta^i|\to\infty}\ell_{\theta}^i(\omega_{t}^i) = \frac{P^i(\omega_{t}^i|\boldsymbol{\phi}_\theta^i)}{\hat{P}^i(\omega_{t}^i|\boldsymbol{\omega}_{1:t-1}^i)}.
\end{align}
Then, taking the expected value of~\eqref{eq:ell_r_infty} results in 
\begin{align}
 &\mathbb{E}_{Q^i}[\log(\ell_{\theta}^i(\omega_{t}^i))] = \int_{\boldsymbol{\Omega}} Q^i(\omega)\log\left(\frac{P^i(\omega|\boldsymbol{\phi}_\theta^i)}{\hat{P}^i(\omega_{t}^i|\boldsymbol{\omega}_{1:t-1}^i)}\right)d\omega. \nonumber
\end{align}
After adding and subtracting $Q^i(\omega)\log(Q^i(\omega))$ inside the integral, we achieve
\begin{align*}
\mathbb{E}_{Q^i}[\log(\ell_{\theta}^i(\omega))] =  D_{KL}(Q^i\|\hat{P}^i(\omega_{t}^i|\boldsymbol{\omega}_{1:t-1}^i))- \nonumber \\
 D_{KL}(Q^i\|P^i(\omega|\boldsymbol{\phi}_\theta^i))).
\end{align*}
Then, from condition (a) of Proposition~\ref{lem:normal_posterior}, our desired result is achieved since $\hat{P}^i(\omega_{t}^i|\boldsymbol{\omega}_{1:t-1}^i)$ converges in probability to $P^i(\cdot|\boldsymbol{\phi}_{\theta^*}^i)$ as $t\to\infty$. 
\end{proof}

Lemma~\ref{cor:exp_log_ell} indicates that as time $t$ becomes very large, the accumulation of likelihood updates $\log(\ell_{\theta}^i(\omega))$ can be approximated as $\exp(t(D_{KL}(Q^i\|P^i(\cdot|\boldsymbol{\phi}_{\theta^*}^i)-D_{KL}(Q^i\|P^i(\cdot|\boldsymbol{\phi}_\theta^i)) + \epsilon))$ for some $\epsilon>0$, where $\epsilon\to0$ as $t\to\infty$. This means that if $D_{KL}(Q^i\|P^i(\cdot|\boldsymbol{\phi}_\theta^i)) >\epsilon+D_{KL}(Q^i\|P^i(\cdot|\boldsymbol{\phi}_{\theta^*}^i)$, then the beliefs will decrease exponentially based on the KL divergence. 

Next, we provide the final property of the uncertain likelihood update that is necessary to prove our main result. 
\begin{lemma} \label{lem:ell_bounded}
Let Assumptions~\ref{assum:support} and~\ref{assum:reg} hold. Then, the uncertain likelihood update is finite and strictly positive value with probability 1 for any $t\ge1$ and all $\theta\in\boldsymbol{\Theta}$ and $i\in\mathcal{M}$, i.e., $0<\ell_{\theta}^i(\omega_{t}^i)<\infty$. 
\end{lemma}
\begin{proof}
First, consider the condition when $|\mathbf{r}_\theta^i|<\infty$ and $t<\infty$. Then, under the regularity conditions, the prior $f_0(\boldsymbol{\phi}^i)$ has full support over the parameter space $\boldsymbol{\Phi}^i$ and the posterior distributions $f(\boldsymbol{\phi}^i|\psi(\{\boldsymbol{\omega}_{1:t-1},\mathbf{r}_\theta\}))$ and $f(\boldsymbol{\phi}^i|\psi(\boldsymbol{\omega}_{1:t-1}^i))$ have not converged to a delta function. Given that the support of the likelihood and the prior are consistent as stated in Assumption~\ref{assum:support}, the support of the posteriors $f(\boldsymbol{\phi}^i|\psi(\{\boldsymbol{\omega}_{1:t-1},\mathbf{r}_\theta\}))$ and $f(\boldsymbol{\phi}^i|\psi(\boldsymbol{\omega}_{1:t}^i))$ remain intact and the posterior predictive distributions $\hat{P}^i(\omega_t^i|\{\boldsymbol{\omega}_{1:t-1},\mathbf{r}_\theta\})=\int_{\boldsymbol{\Phi}}P^i(\omega_t^i|\boldsymbol{\phi}^i)f(\boldsymbol{\phi}^i|\psi(\{\boldsymbol{\omega}_{1:t-1},\mathbf{r}_\theta\}))d\boldsymbol{\phi}^i$ and $\hat{P}^i(\omega_t^i|\boldsymbol{\omega}_{1:t-1}^i)=\int_{\boldsymbol{\Phi}}P^i(\omega_t^i|\boldsymbol{\phi}^i)f(\boldsymbol{\phi}^i|\psi(\boldsymbol{\omega}_{1:t-1}^i))d\boldsymbol{\phi}^i$ are finite and nonzero. Thus, $0<\ell_\theta^i(\omega_{t}^i)<\infty$ for all $\theta\in\boldsymbol{\Theta}$.

When the number of observation grows unboundedly and there is a finite amount of prior evidence, i.e., $t\to\infty$ and $|\mathbf{r}_\theta^i|<\infty$, the uncertain likelihood update converges to $\lim_{t\to\infty} \ell_{\theta}^i(\omega_{t}^i) = 1$ according to Lemma~\ref{lem:ulru_finite}. Furthermore, when the agent $i$ becomes certain and the number of observations is finite, i.e., $|\mathbf{r}_\theta^i|\to\infty$ and $t<\infty$, the uncertain likelihood update is $\ell_{\theta}^i(\omega_{t}^i)=P^i(\omega_{t}^i|\boldsymbol{\phi}_\theta^i)/\hat{P}^i(\omega_{t}^i|\boldsymbol{\omega}_{1:t-1}^i)$. As stated above, the posterior predictive distribution $\hat{P}^i(\omega_{t}^i|\boldsymbol{\omega}_{1:t-1}^i)$ is finite and non zero and $P^i(\omega_{t}^i|\boldsymbol{\phi}_\theta^i)$ is finite and non zero according to Assumption~\ref{assum:support}. Thus, combining the above three conditions indicates that $0<\ell_{\theta}^i(\omega_{t}^i)<\infty$. 
\end{proof}

Now we can prove Theorem~\ref{cor:dogmatic}.

\begin{proof}[Proof of Theorem~\ref{cor:dogmatic}]
First, suppose that agent $i$ has collected an infinite amount of evidence for a hypothesis $\theta$. Then, from Lemma~\ref{lem:normal_posterior} condition 2, the numerator of the uncertain likelihood update~\eqref{eq:ulu} converges in probability to $P^i(\cdot|\phi_\theta^i)$ as $|\mathbf{r}_\theta^i|\to\infty$. 
When $P^i(\cdot|\phi_\theta^i)=P^i(\cdot|\phi_{\theta^*}^i)$, the result is trivially achieved following the proof of Theorem~\ref{th:LL} above since $\ell_{\theta}^i(\cdot)\overset{p}{\to} 1$ in probability and $0<\ell_{\theta}^i(\omega_{t}^i)<\infty$ for all $t$ and $i\in\mathcal{M}$ according to Corollary~\ref{cor:ell_matching} and Lemma~\ref{lem:ell_bounded} respectively. 

Next, suppose that $P^i(\cdot|\phi_\theta^i)\ne P^i(\cdot|\phi_{\theta^*}^i)$. Then, if we divide the log-beliefs by time $t$ we have
\begin{align}
    \frac{1}{t}\log(\boldsymbol{\mu}_t(\theta)) = \frac{1}{t}\mathbf{A}^t\log(\boldsymbol{\mu}_{0}(\theta)) + \frac{1}{t}\sum_{\tau = 1}^t \mathbf{A}^{t-\tau}\log(\boldsymbol{\ell}_{\theta}(\tau)), \nonumber
\end{align}
where $\boldsymbol{\mu}_0(\theta) = \mathbf{1}$ as stated at the beginning of Section~\ref{sec:nbsl_um}. Then, breaking the summation up into three parts for any $T>0$ results in,
\begin{align} \label{eq:rate_log_mu}
    \frac{1}{t}\log&(\boldsymbol{\mu}_{t+1}(\theta))= \nonumber \\ &\frac{1}{t}\left[\sum_{\tau = 1}^T \mathbf{A}^{t-\tau}\log(\boldsymbol{\ell}_{\theta}(\tau)) + \sum_{\tau = T+1}^{t-T} \mathbf{A}^{t-\tau}\log(\boldsymbol{\ell}_{\theta}(\tau)) \right. \nonumber \\
     &  \left. + \sum_{\tau=t-T+1}^{t} \mathbf{A}^{t-\tau}\log(\boldsymbol{\ell}_{\theta}(\tau)) \right]
\end{align}
The first term on the right-hand side of the preceding equation for any finite $T>0$ is bounded above as follows.
\begin{align}
    \frac{1}{t}\sum_{\tau = 1}^T \mathbf{A}^{t-\tau}\log(\boldsymbol{\ell}_{\theta}(\tau)) &\le \frac{1}{t}\sum_{\tau = 1}^T \Big\|\mathbf{A}^{t-\tau}\Big\|_2 \|\log(\boldsymbol{\ell}_{\theta}(\tau))\|_1 \nonumber \\
    &\le \frac{m C_1 T}{t} ,
\end{align}
where $C_1 = \max_{\tau=1,...,T; i\in\mathcal{M}} \log(\ell_{\theta}^i(\omega_{i\tau}))<\infty$ by Lemma~\ref{lem:ell_bounded}, and $\|\mathbf{A}\|_2\le 1$ since $\mathbf{A}$ is doubly stochastic according to Assumption~\ref{assum:graph}. Thus, for any finite $T$, the limit of the first term on the right-hand side of~\eqref{eq:rate_log_mu} deterministically converges in probability to 
\begin{align}
    \lim_{t\to\infty} \frac{1}{t}\sum_{\tau = 1}^T \mathbf{A}^{t-\tau}\log(\boldsymbol{\ell}_{\theta}(\tau)) = 0.
\end{align}

Similarly, the third term on the right-hand side of~\eqref{eq:rate_log_mu} is bounded as follows. \begin{align}
    \frac{1}{t}&\sum_{\tau = t-T+1}^t \mathbf{A}^{t-\tau}\log(\boldsymbol{\ell}_{\theta}(\tau)) \nonumber \\ 
    &\le \frac{1}{t}\sum_{\tau = t-T+1}^t \Big\|\mathbf{A}^{t-\tau}\Big\|_2 \|\log(\boldsymbol{\ell}_{\theta}(\tau))\|_1 \le \frac{m C_2 T}{t} ,
\end{align}
where $C_2 = \max_{\tau=t-T+1,...,t; i\in\mathcal{M}} \log(\ell_{\theta}^i(\omega_{i\tau}))<\infty$ by Lemma~\ref{lem:ell_bounded}. Thus, the limit of the above deterministically converges in probability to 
\begin{align}
    \lim_{t\to\infty} \frac{1}{t}\sum_{\tau = t-T+1}^t \mathbf{A}^{t-\tau}\log(\boldsymbol{\ell}_{\theta}(\tau)) \jzh{=} 0.
\end{align}

Next, we consider the second term on the right-hand side of~\eqref{eq:rate_log_mu}. First, we define the following quantities 
\begin{align}
    \mathbf{e}_\ell(\tau) &= \log(\boldsymbol{\ell}_{\theta}(\tau)) - \log\left(\mathcal{L}^\theta_{\theta^*}\right) \nonumber \\
    \mathbf{e}_\lambda(\tau)& = \mathbf{A}^{t-\tau} - \frac{1}{m}\mathbf{11'},
\end{align}
where $\mathcal{L}^\theta_{\theta^*} = \mathbf{P}_\theta \oslash \mathbf{P}_{\theta^*}$ is the likelihood ratio of $\mathbf{P}_\theta = [P^1(\cdot|\phi^1(\theta)),...,P^m(\cdot|\phi^m(\theta))]'$ and $\mathbf{P}_{\theta^*} = [P^1(\cdot|\phi^1(\theta^*)),...,P^m(\cdot|\phi^m(\theta^*))]'$; and $\oslash$ is the Hadamard division. Then, the second term becomes
\begin{align} \label{eq:second}
    \frac{1}{t}&\sum_{\tau = T+1}^{t-T} \mathbf{A}^{t-\tau}\log(\boldsymbol{\ell}_{\theta}(\tau))  \nonumber \\
    &= \frac{1}{t} \sum_{\tau=T+1}^{t-T} \left(\frac{1}{m}\mathbf{11'}+\mathbf{e}_\lambda(\tau)\right)\left( \log\left(\mathcal{L}^\theta_{\theta^*}\right) + \mathbf{e}_\ell(\tau)\right) \nonumber \\
     & \le \frac{1}{t}\sum_{\tau=T+1}^{t-T} \frac{1}{m}\sqrt{2}m\lambda^{t-\tau} \|\mathbf{e}_\ell(\tau)\| + \frac{1}{m}\mathbf{11'}\|\mathbf{e}_\ell(\tau)\| \nonumber \\ 
     &  + \sqrt{2}m\lambda^{t-\tau} \bigg\|\log\left(\mathcal{L}^\theta_{\theta^*}\right)\bigg\| + \mathbf{11'}\log\left(\mathcal{L}^\theta_{\theta^*}\right), 
\end{align}
where the upper bound is achieved by expanding the second line and taking the spectral norm of the second third and fourth terms and applying Lemma~\ref{lem:doubly}.

Then, we can take the limit of~\eqref{eq:second} as $t\to\infty$. Since $\lim_{t\to\infty}e_{\ell}^i(t) = 0$ in probability for all $i\in\mathcal{M}$ according to Lemma~\ref{lem:ulru_finite} and $0<\lambda<1$, it follows from Lemma~\ref{lemm:ram} that the first term on the right-hand side of~\eqref{eq:second} converges to zero. 

Then, let $\epsilon>0$ be arbitrary. Since $\lim_{t\to\infty}e_{\ell}^i(t) = 0$ $\forall i\in\mathcal{M}$, there exists a time $T$ s.t. for all $t\ge T$, $\mathbf{e}_{\ell}(t)\le \epsilon\mathbf{1}$. Thus, we can bound the second term on the right-hand side of~\eqref{eq:second} in probability by 
\begin{align}
    \lim_{t\to\infty}\frac{1}{t}\sum_{\tau=T+1}^{t-T}\frac{1}{m}\mathbf{11'}\|\mathbf{e}_\ell(\tau)\| \le \lim_{t\to\infty}\frac{\epsilon (t-2T)}{t} = \epsilon. 
\end{align}

Additionally, $T$ is chosen sufficiently large s.t. $\lambda^T \le \epsilon$ allowing the third term on the right-hand side of~\eqref{eq:second} to be bounded by
\begin{align}
    \lim_{t\to\infty}&\frac{1}{t}\sum_{\tau=T+1}^{t-T}\sqrt{2}m\lambda^{t-\tau} \bigg\|\log\left(\mathcal{L}^\theta_{\theta^*}\right)\bigg\| \nonumber \\
    &\le \lim_{t\to\infty}\frac{\sqrt{2}m\epsilon(t-2T)}{t}\sum_{\tau=T+1}^{t-T} \sum_{i=1}^m\bigg|\log\left(\frac{P^i(\cdot|\phi_\theta^i)}{P^i(\cdot|\phi_{\theta^*}^i)}\right)\bigg| \nonumber \\
    & = \sqrt{2}m\epsilon C_3,
\end{align}
where 
\begin{align}
    C_3&=\sum_{i=1}^m \int_{\boldsymbol{\Omega}} Q^i(\omega)\bigg|\log\left(\frac{P^i(\omega|\phi_{\theta^*}^i)}{P^i(\omega|\phi_\theta^i)}\right)\bigg|d\omega \nonumber \\
    & = \sum_{i=1}^m \left[ \int_{\boldsymbol{\Omega}_-} Q^i(\omega)\log\left(\frac{Q^i(\omega)}{P^i(\omega|\phi_\theta^i)}\right)d\omega \right.\nonumber \\
    &\ \ \ \ \ \ \  -\int_{\boldsymbol{\Omega}_-} Q^i(\omega)\log\left(\frac{Q^i(\omega)}{P^i(\omega|\phi_{\theta^*}^i)}\right)d\omega \nonumber \\
    & \ \ \ \ \ \ \ - \int_{\boldsymbol{\Omega}_+} Q^i(\omega)\log\left(\frac{Q^i(\omega)}{P^i(\omega|\phi_\theta^i)}\right)d\omega \nonumber \\
    & \ \ \ \ \ \ \ \left.+ \int_{\boldsymbol{\Omega}_+} Q^i(\omega)\log\left(\frac{Q^i(\omega)}{P^i(\omega|\phi_{\theta^*}^i)}\right)d\omega \right] \nonumber
\end{align}
by the strong law of large numbers; and $\boldsymbol{\Omega}_+ = \{\omega\in\boldsymbol{\Omega}|P^i(\omega|\phi_\theta^i)\ge P^i(\omega|\phi_{\theta^*}^i)\}$ and $\boldsymbol{\Omega}_- = \{\omega\in\boldsymbol{\Omega}|P^i(\omega|\phi_\theta^i)< P^i(\omega|\phi_{\theta^*}^i)\}$ are the sets of observations s.t. the log-likelihood ratio $\log(P^i(\omega|\phi_\theta^i)/P^i(\omega|\phi_{\theta^*}^i))$ is positive or negative respectively. Then, by Assumption~\ref{assum:support} and since the probability of agent $i$ drawing an $\omega\in\boldsymbol{\Omega}$ s.t. $Q^i(\omega)=0$ is zero, the constant $C_3$ is finite for any $m<\infty$. Thus, the third term on the right-hand side of~\eqref{eq:second} is finite.

The final term on the right-hand side of~\eqref{eq:second} simply converges almost surely to the KL divergence between the two distributions as follows. 
\begin{align}
    \lim_{t\to\infty}&\frac{1}{t}\sum_{\tau=T+1}^{t-T}\frac{1}{m}\mathbf{11'}\log\left(\mathcal{L}^\theta_{\theta^*}\right)\nonumber \\
    & = \lim_{t\to\infty}\frac{t-2T}{m t} \sum_{i=1}^m  \frac{1}{(t-2T)}\sum_{\tau=T+1}^{t-T}\log\left(\frac{P^i(\cdot|\phi_\theta^i)}{P^i(\cdot|\phi_{\theta^*}^i)}\right) \nonumber \\
    & = \frac{1}{m} \sum_{i=1}^m  \mathbb{E}_{Q^i}\left[\log\left(\frac{P^i(\cdot|\phi_\theta^i)}{P^i(\cdot|\phi_{\theta^*}^i}\right)\right] \nonumber \\
    & = \frac{1}{m}\sum_{i=1}^m D_{KL}\Big(Q^i\Big\|P^i(\cdot|\phi_{\theta^*}^i))\Big) -  D_{KL}\Big(Q^i\Big\|P^i(\cdot|\phi_\theta^i))\Big),
\end{align}

Thus, putting it all together results in
\begin{align}
    \lim_{t\to\infty}& \frac{1}{t}\log(\boldsymbol{\mu}_t(\theta)) \nonumber \\
    &\le \epsilon( 1+\sqrt{2}mC_3) + \frac{1}{m}\sum_{i=1}^m D_{KL}\Big(Q^i\Big\|P^i(\cdot|\phi_{\theta^*}^i))\Big) \nonumber \\
    & \ \ \ - D_{KL}\Big(Q^i\Big\|P^i(\cdot|\phi_\theta^i))\Big).
\end{align}
Since, $D_{KL}(Q^i\|P^i(\cdot|\phi_\theta^i)) > D_{KL}(Q^i\|P^i(\cdot|\phi_{\theta^*}^i))$ for at least one agent $i$ with parameters for hypothesis $\theta$ $\phi_\theta^i\ne\phi_{\theta^*}^i$. Therefore, we can select a finite time $T$ such that $\epsilon$ is arbitrarily small and $(1/m)\sum_{i=1}^m D_{KL}(Q^i\|P^i(\cdot|\phi_\theta^i))>\epsilon( 1+\sqrt{2}mC_3)+(1/m)\sum_{i=1}^m D_{KL}(Q^i\|P^i(\cdot|\phi_{\theta^*}^i))$. This causes the log-beliefs to diverge in probability to $\log(\mu_{t}^i(\theta))\to -\infty$ for all $i\in \mathcal{M}$. Then, by the continuity of the exponential function, the beliefs converge in probability to $\lim_{t\to\infty} \mu_{t}^i(\theta)=0$ for all $i\in\mathcal{M}$.

\end{proof}

\section{Regularity conditions} \label{app:regularity}
\begin{assumption}[Regularity Conditions \cite{W1969,KV2012}] \label{assum:reg_sup}
We assume the following regularity conditions for the likelihood functions $P(\cdot|\phi)$ and the prior distribution $f_0(\phi)$. 
\begin{enumerate}[label=(\alph*)]
    \item The set of possible parameters $\boldsymbol{\Phi}$ is defined on a compact set. 
    \item If $\phi_1$ and $\phi_2$ are two distinct points in $\boldsymbol{\Phi}$, then the distributions $P(\cdot|\phi_1)$ and $P(\cdot|\phi_2)$ are different. 
    \item The parameters $\{\phi_{\theta}| \theta\in\boldsymbol{\Theta}\}$ lie within $\boldsymbol{\Phi}$. 
    \item The prior distribution $f_0(\phi)$ is continuous everywhere and nonzero for all $\{\phi_{\theta}|\theta\in\boldsymbol{\Theta}\}$, i.e., $f_0(\phi_\theta)>0$.
    \item The ground truth parameter $\phi_{\theta^*}$ uniquely minimizes the KL divergence, $\phi_{\theta^*} = \argmin_{\phi} D_{KL}(Q\|P(\cdot|\phi))$.
    \item The log-likelihood $\log(P(\cdot|\phi))$ is twice differentiable with respect to $\phi$ in the neighborhood of $\phi_{\theta^*}$ and $\phi_{\theta}$.
    \item The Fisher information matrices $J(\phi_{\theta^*})$ and $J(\phi_{\theta})$ of the log-likelihood $\log(P(\cdot|\phi))$ evaluated at $\phi_{\theta^*}$ and $\phi_{\theta}$ are strictly positive and bounded. 
\end{enumerate}
For a detailed list and discussion of the regularity conditions, please see \cite{W1969} for well-specified likelihood models and \cite{KV2012} when the likelihood models are misspecified. 
\end{assumption}

\end{document}